\documentclass{article}

\usepackage{arxiv}

\usepackage[utf8]{inputenc} 
\usepackage[T1]{fontenc}    
\usepackage{hyperref}       
\usepackage{url}            
\usepackage{booktabs}       
\usepackage{amsfonts}       
\usepackage{nicefrac}       
\usepackage{microtype}      
\usepackage{graphicx}
\usepackage{natbib}
\usepackage{doi}

\usepackage{apxproof}
\usepackage{amsmath,amssymb, array}
\usepackage{amsthm}
\usepackage{algorithmic}
\usepackage{textcomp}
\usepackage{xcolor}
\usepackage{verbatim}
\usepackage{bm}
\usepackage{pifont}

\newcommand{\extremaVgl}{magnitude-comparable} 
\newcommand{\minVgl}{unbiased-trustworthy}
\newcommand{\skewsensitive}{skew-sensitive}
\newcommand{\stereosensitive}{stereotype-sensitive}
\newcommand{\wordbias}{word bias}
\newcommand{\setbias}{set bias}
\newcommand{\skewbias}{skew}
\newcommand{\stereobias}{stereotype}

\newcommand{\R}{\mathbb{R}}

\DeclareMathOperator*{\EX}{\mathbb{E}}

\newcommand{\weatw}{$\textit{WEAT}_\textit{word}$}
\newcommand{\weat}{$\textit{WEAT}$}
\newcommand{\macw}{$\textit{MAC}_\textit{word}$}
\newcommand{\mac}{$\textit{MAC}$}
\newcommand{\dbw}{$\textit{DirectBias}_\textit{word}$}
\newcommand{\db}{$\textit{DirectBias}$}
\newcommand{\samew}{$\textit{SAME}_\textit{word}$}
\newcommand{\same}{$\textit{SAME}$}

\title{Evaluating Metrics for Bias in Word Embeddings}

\author{Sarah Schröder \and \bf{Alexander Schulz} \and \bf{Philip Kenneweg} \and \bf{Robert Feldhans} \and \bf{Fabian Hinder} \and \bf{Barbara Hammer}\\
	CITEC, Machine Learning Group\\
	Bielefeld University - Faculty of Technology \\
	\texttt{\{saschroeder, aschulz, pkenneweg, rfeldhans, fhinder, bhammer\}@techfak.uni-bielefeld.de} \\
}

\renewcommand{\shorttitle}{Evaluating Metrics for Bias in Word Embeddings}

\hypersetup{
pdftitle={Evaluating Metrics for Bias in Word Embeddings},
pdfsubject={cs.CL}, 
pdfauthor={Sarah ~Schröder, Alexander Schulz, Philip Kenneweg, Robert Feldhans, Fabian Hinder, Barbara Hammer},
pdfkeywords={Bias Detection, Fairness, Word Embeddings},
}

\begin{document}
\newtheorem{definition}{Definition}
\newtheorem{theorem}{Theorem}
\newtheorem{lemma}{Lemma}

\newtheoremrep{theorem}{Theorem}
\newtheoremrep{lemma}{Lemma}

\maketitle

\begin{abstract}
Over the last years, word and sentence embeddings have established as text preprocessing for all kinds of NLP tasks and improved the performances significantly. Unfortunately, it has also been shown that these embeddings inherit various kinds of biases from the training data and thereby pass on biases present in society to NLP solutions.\\
Many papers attempted to quantify bias in word or sentence embeddings to evaluate debiasing methods or compare different embedding models, usually with cosine-based metrics. However, lately some works have raised doubts about these metrics showing that even though such metrics report low biases, other tests still show biases. In fact, there is a great variety of bias metrics or tests proposed in the literature without any consensus on the optimal solutions. Yet we lack works that evaluate bias metrics on a theoretical level or elaborate the advantages and disadvantages of different bias metrics.\\
In this work, we will explore different cosine based bias metrics. We formalize a bias definition based on the ideas from previous works and derive conditions for bias metrics. Furthermore, we thoroughly investigate the existing cosine-based metrics and their limitations to show why these metrics can fail to report biases in some cases. Finally, we propose a new metric, SAME, to address the shortcomings of existing metrics and mathematically prove that SAME behaves appropriately.

There are two conference papers including new experiments to evaluate the properties of existing cosine scores (\cite{icpram24schroeder}) and the SAME score (\cite{ijcnn24schroeder}).
\end{abstract}



\section{Introduction}
\label{sec:intro}

Word embeddings have been widely adopted in NLP tasks over the last decade, and have been further developed to embed the context of words, whole sentences or even longer texts into high dimensional vector representations. These are applied in a wide variety of downstream tasks, including translation and text generation as in question-answering systems (\cite{gpt}) or next sentence prediction (\cite{bert}), or sentiment and similarity analysis (\cite{use,word2vec,glove,sentencebert}) of texts.\\
However, many works have shown that these models capture various kinds of biases present in humans and society. Using biased word or text representations poses a great risk to produce unfair language models, which in turn can further amplify the biases in society. \\
The first works investigating bias in word embeddings focused on geometric relations between embeddings of stereotypically associated words. From there, several metrics based on cosine similarity have been proposed like the "Direct Bias" by \cite{bolukbasi} and the Word Embedding Association Test (WEAT) by \cite{weat}. Though initially introduced to prove the presence of bias in word embeddings, these tests have also been widely used to compare biases in different embedding models ((e.g. \cite{seat, use}) or validate debiasing method (e.g. \cite{sentdebias, karve2019conceptor, kaneko2021dictionary}).
Due to its popularity, there have been several variants introduced in the literature like the Sentence Encoder Association Test (SEAT) by \cite{seat}, the "Generalized WEAT" (\cite{weatgeneralized}) or WEAT for gendered languages (\cite{weatgender}).
Another cosine based metric is the Mean Average Cosine Similarity (MAC) score (\cite{mac}), though it has a different intuition than WEAT and the Direct Bias. \\
However, there is literature criticizing WEAT and its variants. For instance \cite{seat} found that SEAT produced rather inconsistent results on sentence embeddings and emphasized the positive predictability of the test. This prompted many other papers to doubt the meaningfulness or usability of cosine based metrics in general. Hence, methods to quantify bias in downstream tasks have been proposed and increasingly replace cosine based metrics (\cite{coreference, lipstick}).\\
On the one hand, this is desirable since it bridges the gap to fairness evaluation in general and allows to apply other fairness definitions like statistical parity or counterfactual fairness. On the other hand, it poses the risk to confound biases from word representations with biases introduced by the downstream task itself. Hence, when investigating bias in word embedding models, apart from specific downstream tasks, metrics directly applicable to the embeddings are beneficial.\\
Nevertheless, the doubts regarding existing metrics show the need to evaluate these metrics on a theoretical level. To our knowledge, the only work doing this is \cite{ripa}. The authors showed that WEAT has theoretical flaws that can cause it to overestimate biases and demonstrated that one could make up attribute sets in order to make the test appear statistically insignificant. While this alone is highly worrying, we find even more theoretical flaws that, on one hand, cause WEAT to underestimate biases under certain conditions. On the other hand, these flaws make WEAT results incomparable between different embeddings. Moreover, these flaws extend to other versions like SEAT or the generalized WEAT, and we also find flaws for MAC and the Direct Bias.\\
Noteworthy, there exist also works criticising that cosine based bias scores do not capture biases entirely. However, they rely on different definitions of bias. As we discuss in Chapter \ref{sec:bias_definition}, these definitions do not necessarily contradict the definition by \cite{bolukbasi} or \cite{weat}.\\
Since there are applications like document clustering or text similarity estimations, where the cosine similarity is directly queried (\cite{sentencebert, use, gao2021simcse, subramanian2018learning, liu2021transencoder}), quantifying bias in terms of cosine similarity is highly relevant and thus further evaluation of cosine based metrics remains necessary.\\

The main contributions in this work are the following:
(i) We define formal requirements for cosine based score functions which formalize a notion of meaningfulness in this context.
(ii) We analytically demonstrate that the state of the art metrics WEAT, the Direct Bias and MAC are only partially meaningful, corresponding to the previous definitions.
(iii) We propose a novel bias score function \emph{Scoring Association Means of Word Embeddings} (SAME), that is based on some principles of WEAT, but reformulates potentially problematic parts.
(iv) We experimentally investigate SAME and the state of the art score functions and substantiate the theoretic claims with real word embeddings.\\

The remainder of this paper is structured as follows: In Chapter \ref{sec:related_work} we give an overview over bias definitions and specifically cosine based bias metrics from the literature. The formal requirements for bias score functions are explained and applied to the existing score functions in Chapter \ref{sec:requirements}. Next, in Chapter 
\ref{sec:proposed_metric}, we propose SAME with extensions for multi-attribute biases and skew and stereotype distinction. Chapter \ref{sec:experiments} describes our experiments, where we retrain BERT on biased data. Thereby we induce specific biases as a ground truth to evaluate the different bias score functions. Lastly, we conclude our findings and give an overview over the bias score functions in Chapter \ref{sec:conclusion}.

\section{Related Work for Bias in Text Embeddings}
\label{sec:related_work}
In the following, we summarize the related work on bias in word embeddings. First, we focus on definitions and discuss our choice of cosine based scores. Then we revisit the different cosine based metrics from the literature.

\subsection{Bias Definition} \label{sec:bias_definition}

In this section, we describe several definitions of bias in word embeddings mentioned in the literature. The focus of this work lies on geometrical bias definitions, especially those regarding cosine similarity, and the skew and stereotype definitions. For the sake of clarity, we also discuss other bias definitions regarding downstream tasks. Lately, measuring bias in downstream tasks increasingly replaces geometrical bias scores in the literature. Yet, we argue that they cannot replace the geometrical definitions and scores entirely and thereby emphasize the usability of cosine based scores.

\subsubsection{Geometrical Bias Definitions}
Word embedding vectors are expected to reflect word relations by their geometrical relations, whereby distance in the vector space is said to correlate with the semantic similarity of words (\cite{word2vec, glove}). Due to this intuition, biases in word embeddings are often considered as flawed vector representations, where words are embedded inappropriately close or far from other words. Since the cosine similarity is often used as a distance measure to estimate word or document similarity (\cite{thongtan2019sentiment, shahmirzadi2019text, sentencebert, use}), it has also been used to measure bias in terms of geometric relations.\\


\begin{figure}[tb]
	\centering
	\includegraphics[scale=0.2]{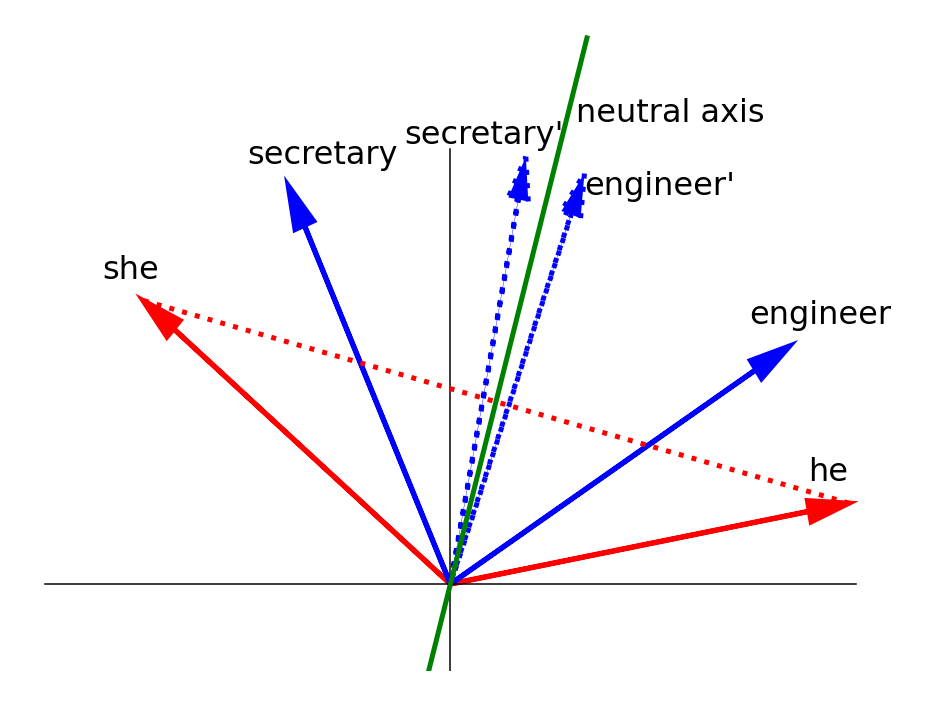}
	\caption{Simplified (2D) visualization of gender stereotypes in occupation embedding vectors. The occupation vectors in dotted lines show a less biased representation of the occupations. Ideally, all gender neutral terms should be represented on the neutral axis, thus being equidistant to $\bm{he}$ and $\bm{she}$. The red dotted line indicates the bias direction $\bm{g} = \pm (\bm{she}-\bm{he})$}
	\label{fig:bias_example}
\end{figure}

Based on this idea, \cite{bolukbasi} defined bias as the correlation between neutral words and a bias subspace or bias direction, that describes the relations between groups of words. For the simple case illustrated in Figure \ref{fig:bias_example} this would be the gender direction $\bm{g} = \bm{she} - \bm{he}$. To get a more robust estimation of group relations (e.g. male and female terms), one would use several word pairs and determine a low dimensional subspace that captures most of the variance of individual bias directions.\\
\cite{weat} define bias as the difference in angular distance between a neutral word and two different groups of words. In the example depicted in Figure \ref{fig:bias_example}, they would consider $\bm{secretary}$ biased since it has a lower angular distance to $\bm{she}$ than $\bm{he}$ and $\bm{engineer}$ vice versa. \\
Both definitions follow a similar intuition and use the cosine similarity as correlation or distance measure. For the simple case illustrated in Figure \ref{fig:bias_example} correlation with the bias direction matches up with the difference in angular distances. An important difference is that \cite{weat} contrast over two groups of theoretically neutral words to determine a stereotype (see Section \ref{sec:cosine_scores} for details) while \cite{bolukbasi} take the mean bias correlation over all neutral words (see Section \ref{sec:Bolukbasi-directbias}).\\

Another work by \cite{ripa} follows the definition from \cite{bolukbasi}. However, they argue that the vector length holds important information and thus biases should rather be measured using the inner product of word vectors instead of cosine similarity. Yet, their definition of bias mostly concerns word co-occurrences in the training data and how these are represented in word vectors instead of the fairness implications in different use cases of word embeddings. Hence, it is not clear to what extent the information encoded in the vector length influences fairness in downstream tasks. On the other hand, there are downstream tasks such as document clustering that might rely directly on the cosine similarity or use it as a dissimilarity measure. Consequently, using the cosine similarity (as a modification of the inner product that removes the influence of vector lengths) is a valid option in cases where the vector length can be neglected. Moreover, they only consider single word biases. When computing the bias over a set of words, it seems preferable to normalize the word biases by word vector length to have all words contribute with the same weight to the overall bias, resulting in the same notion as if using cosine based scores.\\ 

\subsubsection{Classification and Clustering Bias}

\cite{lipstick} found that stereotypical groups persisted after removing correlations of neutral words with a bias subspace or retraining on debiased text sources. This raised doubts on the usability of cosine based metrics. Instead they proposed to use clustering or classification tasks to investigate bias. In short, if a classifier or clustering algorithm can reconstruct stereotypical groups (according to previous correlation with the bias subspace), then there is still bias in the data, according to their definition.
However, there are two possible explanations why they found this persisting bias: Either the debiasing methods were not performed thoroughly (leaving some geometrical bias associations in the embeddings or text sources) or as \cite{ripa} described it: "Because we define unbiasedness with respect to a set of word pairs, we cannot make any claims about word pairs outside that set".
Another possibility is that these stereotypical groups simply reflect other relations independent of the bias attributes, but different enough to reconstruct stereotypical groups with a flawed classification task (e.g. different types of occupations that coincidentally relate to gender stereotypes). Since it is not clear whether one of these effects lead to their claims, we cannot take \cite{lipstick}'s work as evidence against cosine based metrics.\\

\subsubsection{Bias in other Downstream Tasks}
Lately, more and more works started to measure bias in downstream tasks as opposed to geometrical metrics. An example are \cite{coreference} and \cite{DBLP:journals/corr/abs-1904-03310-zhao1}, who used a co-reference resolution task. For instance, the language model has to decide which word a gendered pronoun was associated with. Confronting the model with pro- and anti-stereotypical sentences can be used to determine whether the decisions are based on gender biases as opposed to semantic meaning.\\
Another work to be emphasized is \cite{kurita}, who proposed a bias test based on the masked language objective of BERT. Essentially, they query BERT for the probability to insert certain bias attributes in a masked sentence when associated with other neutral words. For instance they would compute the association between "programmer" and the male gender by probing BERT for the probability to replace "[MASK]" with "he" in "[MASK] is a programmer".\\
They show that this is an accurate way to measure biases in BERT and can show a correlation with bias in the GPR task by \cite{gap}.\\
However, the bias measured by \cite{kurita} is model-specific for BERT's masked-language modeling (MLM) models and measuring bias in downstream tasks is always task specific. Both might be influenced by fine-tuning for MLM or the specific task. This is fine as long as one wants to quantify bias with respect to this specific model or task, but might not be reliable when comparing embedding models or probing embedding models without a specific use case in mind. On the other hand, cosine based metrics work for any kind of embedding models, from classical word embeddings to contextualized and sentence embeddings, and hence are worth further investigating.

\subsubsection{Skew and Stereotype}
Recently, \cite{skew} introduced another definition of bias. They also used the GPR task by \cite{gap} to obtain word biases, but distinguish between two types of bias: Skew and stereotype. While stereotype shows the presence of groups stronger associated with different bias attributes than other groups, which is also the intuition behind \cite{weat}'s work, the skew describes the effect that all words from a set are (on average) biased towards the same attribute. For instance, occupations could be stereotypically associated with male/female terms, but it's also possible that (on average) all occupations are rather biased towards one gender. They further suggest that there is a certain trade-off between these two forms of bias.\\
Distinguishing these kinds of biases seems to be beneficial for investigating bias in embeddings more thoroughly. Hence, we apply their definition to construct cosine based skew and stereotype metrics (see Section \ref{sec:skew_stereo}) and show their benefit over general bias metrics in the experiments.

\subsection{Cosine-based measure}
\label{sec:cosine_scores}
Cosine-based proximity measures are the most used measures in terms of geometric relations. While it is possible to refer to other similarity/dissimilarity measures, mostly the cosine similarity
\begin{eqnarray}
\label{eq:cos}
cos(\bm{u}, \bm{v}) &=& \frac{\bm{u} \cdot \bm{v}}{||\bm{u}||\cdot||\bm{v}||}.
\end{eqnarray}
is used to determine the association between two vector representations of words $\bm{u}$ and $\bm{v}$.\\
Usually the similarity of words $\bm{w}$ towards groups of words (e.g.\ terms specific for one gender/ religion/ race...), so called attribute sets, is measured. The similarity of $\bm{w}$ and one set of attributes $A$ is then
\begin{eqnarray}
\label{eq:attr_sim}
s(\bm{w},A) &=& mean_{\bm{a} \in A} \cos(\bm{w},\bm{a}).
\end{eqnarray}

Based on this, different bias metrics are defined in literature. We detail the most prominent ones in the following sections, using the notations from the literature.

\subsubsection{WEAT} \label{sec:weat}

The Word Embedding Association Test (WEAT), as defined by \cite{weat}, is based on the Implicit Association Test (\cite{greenwald}), which is used in psychological studies to measure reaction times regarding pro and anti stereotypical associations. \\
The test compares two sets of target words $X$ and $Y$ with two sets of bias attributes $A$ and $B$ of equal size $n$. When observing gender stereotypes in occupations, which might be $A = \{he, man, male, ...\}$, $B = \{she, woman, female, ...\}$, $X = \{engineer, doctor, police man, ...\}$ and $Y = \{secretary, nurse, teacher, ...\}$, anticipating that the occupations in $X$ would be closer to words in $A$ than $B$, and vice versa for occupations in $Y$.\\
The association of a single word $\bm{w}$ with the bias attribute sets $A$ and $B$ including $n$ attributes each, is given by
\begin{eqnarray}
\label{eq:weat_attr_sim}
s(\bm{w},A,B) &=& \frac{1}{n} \sum_{\bm{a} \in A}\cos(\bm{w},\bm{a}) - \frac{1}{n} \sum_{\bm{b} \in B}\cos(\bm{w},\bm{b}).
\end{eqnarray}
To quantify bias in the sets $X$ and $Y$, the effect size is used, which is a normalized measure for the association difference between the target sets
\begin{eqnarray}
\label{eq:weat_eff_size}
d(X,Y,A,B) &=& \frac{mean_{\bm{x} \in X} s(\bm{x},A,B) - mean_{\bm{y} \in Y} s(\bm{y},A,B)}{stddev_{\bm{w} \in X \cup Y} s(\bm{w},A,B)},
\end{eqnarray}
where $mean_{\bm{x} \in X} s(\bm{x},A,B)$ refers to the mean of $s(\bm{x},A,B)$ with $\bm{x}$ in $X$ and $stddev_{\bm{x} \in X} s(\bm{x},A,B)$ to the standard deviation over all word biases of $\bm{x}$ in $X$.\\
A positive effect size confirms the hypothesis that words in $X$ are rather stereotypical for the attributes in $A$ and words in $Y$ stereotypical for words in $B$, while a negative effect size indicates that the stereotypes would be counter-wise.\\
Based on the Implicit Association Tests, the WEATs include several tests probing for associations between pleasant/unpleasant words and race as well as tests probing for gender stereotypes comparing career and family, math and arts as well as science and arts. To determine the statistical significance of biases measured in those tests, the authors use the test statistic
\begin{eqnarray}
\label{eq:weat_test_stat}
s(X,Y,A,B) &=& \sum_{\bm{x} \in X} s(\bm{x},A,B) - \sum_{\bm{y} \in Y} s(\bm{y},A,B)
\end{eqnarray}
for a permutation test with partitions $(X_i,Y_i)$ of $X \cup Y$:
\begin{eqnarray}
\label{eq:weat_perm}
p &=& P_r [ s(X_i,Y_i,A,B) > s(X,Y,A,B)].
\end{eqnarray}

There are several extensions to WEAT: For instance WEAT for gendered languages by \cite{weatgender} and the "Generalized WEAT" by \cite{weatgeneralized}, which allows to apply WEAT to more than two different attribute and target groups. \\ \cite{seat} propose the Sentence Encoder Association Test (SEAT), which basically applies WEAT to sentence representations. The sentences are obtained by inserting WEAT terms into simple templates such as "this is a <term>". They also propose two additional tests: The Angry Black Women Stereotype and the Double Bind tests. 

\subsubsection{MAC}
\label{sec:mac}
The Mean Average Cosine Similarity (MAC) score was introduced by \cite{mac} to provide a metric for multi-class fairness problems. They compute the bias of a word $\bm{t}$ towards an attribute set $A_j$ using the cosine distance as the reciprocal of the cosine similarity:
\begin{eqnarray}
\label{eq:mac_attr_sim}
S(\bm{t},A_j) = \frac{1}{N} \sum_{\bm{a} \in A_j} 1 - cos(\bm{t}, \bm{a}).
\end{eqnarray}
\noindent In contrast to WEAT, only one set of target words $T$ is utilized. The bias of words in $T$ towards several bias attribute sets $A = \{A_1, ..., A_n\}$ is then given by the MAC score as
\begin{eqnarray}
\label{eq:mac}
MAC(T,A) = \frac{1}{|T||A|} \sum_{\bm{t} \in T} \sum_{A_j \in A} S(\bm{t}, A_j).
\end{eqnarray}

Equation \eqref{eq:mac} applies a different intuition to WEAT and the Direct Bias (see Figure \ref{fig:bias_example} for the bias definition). However, MAC uses the direct associations instead of contrasting terms like \cite{weat} or measuring the correlation with a contrastive bias direction (or subspace) like \cite{bolukbasi}.
Most works use the MAC score only for multi-class problems that (the original) WEAT cannot be applied to. For instance, \cite{sentdebias} use it in addition to WEAT to evaluate their debiasing method.

\subsubsection{Direct Bias (Bolukbasi)} \label{sec:Bolukbasi-directbias}

\cite{bolukbasi} define the Direct Bias as the correlation of neutral words $\bm{w} \in N$ with a bias direction (in their example gender direction $\bm{g}$):
\begin{eqnarray}
\label{eq:direct_bias}
Direct Bias &=& \frac{1}{|N|} \sum_{\bm{w} \in N} |\cos(\bm{w},\bm{g})|^c 
\end{eqnarray}

with $c$ determining the strictness of bias measurement. The gender direction is either obtained by a gender word-pair e.g. $\bm{g} = \bm{he} - \bm{she}$ or - to get a more robust estimate - it is obtained by computing the first principal component over a set of individual gender directions from different word-pairs.\\
The authors used it as a preliminary experiment to show bias present in embeddings, but not to evaluate their debiasing algorithm (which follows the same intuition). The Direct Bias is applied e.g. in \cite{genderbias_underrepresentation} and \cite{costa2019evaluating}.\\
In terms of their debiasing algorithm \cite{bolukbasi} describe how to obtain a bias subspace given defining sets $D_1$, ..., $D_n$. A defining set $D_i$ includes words $\bm{w}$ that only differ by the bias relevant topic e.g. for gender bias $\{\bm{man},\bm{woman}\}$ could be used as a defining set. Given these sets, the authors construct individual bias directions $\bm{w} - \bm{\mu_i} \; \forall \bm{w} \in D_i, i \in \{1,...,n\}$ and $\bm{\mu_i} = \sum_{\bm{w} \in D_i} \frac{\bm{w}}{|D_i|}$. To obtain a k-dimensional bias subspace $B$ they compute the $k$ first principal components over these samples.
\cite{mac} applied this debiasing algorithm to mitigate multi-class bias. For this purpose they determined the bias subspace using defining sets $D_i$ with more than two elements. This could be also applied to the Direct Bias to measure multi-class bias.

\section{Requirements for Bias Metrics} 
\label{sec:req}
\subsection{Motivation}
\label{sec:req_motivation}
While WEAT is frequently used in the literature (\cite{weatgender, weatgeneralized, kaneko2021dictionary, sentdebias, karve2019conceptor, nullitout}), there does exist literature documenting its behaviour to be not always as expected, e.g. \cite{seat}, who emphasize that WEAT cannot prove the absence of bias, or \cite{ripa}, who already showed theoretical flaws of WEAT. Apart from this, we show further examples where WEAT, the Direct Bias and MAC fail to report biases.\\

To emphasize the need of formal requirements for bias scores, we look into how WEAT as the most prominent cosine based bias score is used in the literature. The authors (\cite{weat}) introduce WEAT to prove the presence of bias in word embeddings. However, there are other papers that use WEAT as a quantification of bias, either to compare biases in different embeddings (\cite{seat}) or to evaluate debiasing algorithms (\cite{sentdebias, nullitout, karve2019conceptor,  kaneko2021dictionary}). Those use cases require a bias score whose magnitude can be interpreted analogously to more or less bias. However, we present an example that shows that the WEAT effect size cannot be interpreted in that way. Therefor, its meaningfulness for bias quantification is severely limited.

Consider the example of gender bias with $A$ containing embeddings of male terms and $B$ containing those of female terms, and 6 occupations, half of those stereotypically male/female. All words were embedded with BERT. We computed the word-wise biases of 6 occupations according to WEAT's $s(\bm{w},A,B)$. The results are shown in Figure \ref{fig:weat_pretest} (the y-axis is added for better visibility and does not imply any difference). 

\begin{figure}[tb]
	\centering
	\includegraphics[scale=0.3]{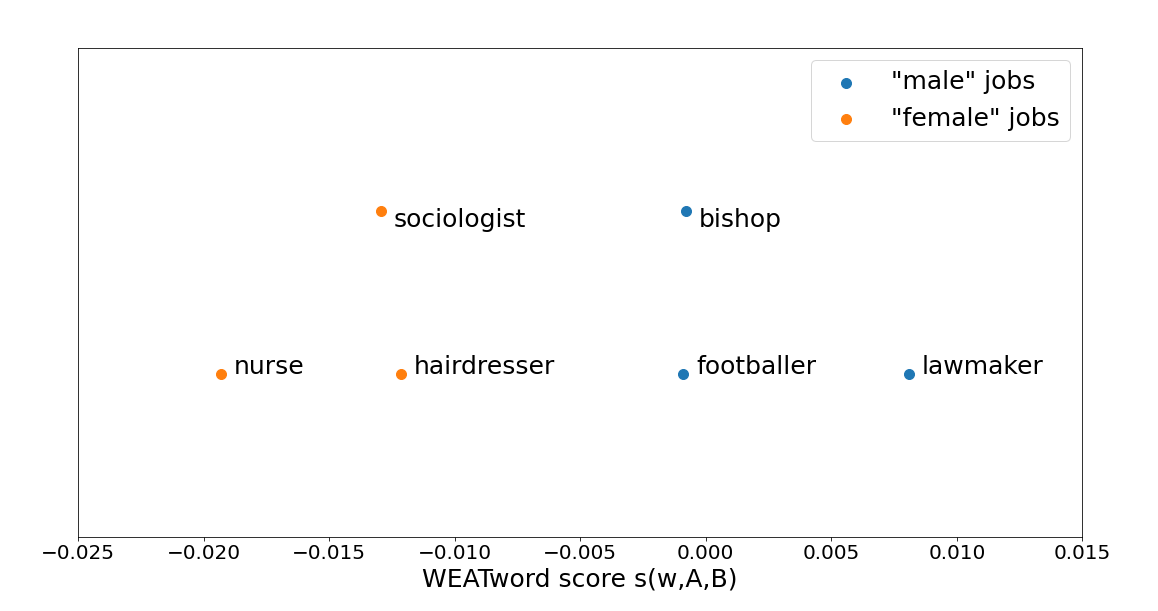}
	\caption{Bias of stereotypical male and female jobs as indicated by WEAT's $s(\bm{w},A,B)$ calculated on BERT embeddings. Differences along the y-axis are arbitrary shift for better visibility. }
	\label{fig:weat_pretest}
\end{figure}

We then used groups $X = \{\bm{footballer},\bm{bishop}\}$ and $Y = \{\bm{hairdresser}, \bm{sociologist}\}$ as target words and measured the effect size $d(X,Y,A,B) = 1.7299$, which indicates a high stereotype regarding the groups $X$ and $Y$.\\
Furthermore, we constructed groups $X' = \{\bm{footballer},\bm{lawmaker}\}$ and $Y' = \{\bm{hairdresser}, \bm{nurse}\}$, using the more extremely stereotyped occupations according to the word-wise bias scores. Now, when computing the effect size using $X'$ and/or $Y'$ instead of their counterparts, one would expect a bias metric to report an even larger bias with respect to these groups.\\
However, we found $d(X,Y',A,B) = -0.0865$ and $d(X',Y',A,B) = 1.6448e-05$, which would be interpreted as significantly lower biases, although looking at individual words we know that this is not the case.\\
This is just one example to illustrate how WEAT as the most prominent cosine based bias score can be misleading when quantifying biases. Hence, in the following sections, we derive formal requirements for bias scores to ensure their results are reliable and comparable.

\subsection{Formal Bias Definition and Notations}
\label{sec:formal_bias_def}
In Section \ref{sec:bias_definition} we described the geometrical definition of bias that is commonly used in the literature (e.g. \cite{weat} and \cite{bolukbasi}). However, this definition considers only the simple case of comparing one neutral word with two other words representing two protected groups. Thus, in this section, we will provide a formal definition of bias covering arbitrary numbers of words and groups to compare against.\\

Given a fairness critical concept like gender, religion or ethnicity, we select $n$ protected groups that might be subject to biases. Each protected group is defined by a set of attribute words $\bm{a_{ik}} \in A_i$ with $i \in \{1, ..., n\}$ the group's index. We summarize these attribute sets as $A = \{A_1, ..., A_n\}$. The intuition is that the attributes define the relation of protected groups by contrasting specifically over the membership to the different groups. Therefore, it is important that any attribute $\bm{a_{ik}} \in A_i$ has a counterpart $\bm{a_{jk}} \in A_j \; \forall \; A_j \in A, j \neq i$ that only differs from $\bm{a_{ik}}$ by the group membership. For instance, if we used $A_1 = \{she, female, woman\}$ as a selection of female terms, $A_2 = \{he, male, man\}$ would be the proper choice of male terms.

Given these attribute sets the association of a words $\bm{w}$ with a protected group is given by $s(\bm{w},A_i)$ (Equation \ref{eq:attr_sim}). We are particularly interested in the difference of associations towards the different groups, i.e. is $\bm{w}$ more similar to one protected groups than the others. Whether such associations are harmful depends on whether $\bm{w}$ is theoretically neutral to the protected groups. For example, terms like "aunt" or "uncle" are associated with one or the other gender per definition, while a term like "transsexual" is associated with the overall concept of gender but should not be stronger associated with female terms than male terms.

Noteworthy, by words and attributes we refer to vectorial representations of words in a $d$-dimensional embedding space, i.e. let $V = \{v_1, ..., v_m\}$ be a vocabulary, then an embedding model $E : V \rightarrow \mathbb{R}^d$ transforms words $v$ into $d$-dimensional word vectors $\bm{w}$.

As baseline for our bias score requirements and the following analysis of bias scores from the literature, we suggest two intuitive definitions of bias for single words $\bm{w}$ and sets of words $W$.
For words $\bm{w}$ we apply the intuition of WEAT extended to $n$ protected groups instead of only two.

\begin{definition}[\wordbias]
	\label{def:word_bias}
	Given $n$ protected groups represented by attribute sets $A_1$, ..., $A_n$ and a word $\bm{w}$ that is theoretically neutral to these groups, we consider $\bm{w}$ biased if 
	\begin{align}
	\label{eq:word_bias}
    \exists A_i, A_j \in A: s(\bm{w},A_i) > s(\bm{w},A_j)
	\end{align}
\end{definition}

\begin{definition}[\setbias]
	\label{def:set_bias}
	Given $n$ protected groups represented by attribute sets $A_1$, ..., $A_n$ and a set of word $W$ containing only words that are theoretically neutral to these groups, we consider $W$ biased if at least one word $\bm{w} \in W$ is biased:
	\begin{align}
	\label{eq:set_bias}
    \exists A_i, A_j \in A, \bm{w} \in W: s(\bm{w},A_i) > s(\bm{w},A_j)
	\end{align}
\end{definition}

The idea behind Definition \ref{def:set_bias} is that even when looking at sets of words, each individual word's bias is important, i.e. as long as there is one biased word in the set, we cannot call the set of words unbiased, even if word biases cancel out on average or the majority of words is unbiased.

Apart from this general notion of bias, we further introduce the definition of skew and stereotype bias from \cite{skew} to geometrical biases. These two types of bias describe how a set of words $W$ is biased. The skew describes whether all words are (on average) biased towards on protected group, while the stereotype describes the occurrence of stereotypical groups associated with different protected groups. 

\begin{definition}[\skewbias]
	\label{def:skew_bias}
	Given $n$ protected groups represented by attribute sets $A_1$, ..., $A_n$ and a set of word $W$, we consider $W$ skewed towards a group $i$ compared to group $j$ if
	\begin{align}
	\label{eq:skew_bias}
    mean_{\bm{w} \in W} s(\bm{w},A_i) > mean_{\bm{w} \in W} s(\bm{w},A_j)
	\end{align}
\end{definition}

\begin{definition}[\stereobias]
	\label{def:stereo_bias}
	Given $n$ protected groups represented by attribute sets $A_1$, ..., $A_n$ and a set of word $W$, we consider $W$ containing stereotypes if:
	\begin{align}
	\label{eq:stereo_bias}
    \exists A_i, A_j \in A, \bm{w_1}, \bm{w_2} \in W: s(\bm{w_1},A_i) - s(\bm{w_1},A_j) \neq s(\bm{w_2},A_i) - s(\bm{w_2},A_j)
	\end{align}
\end{definition}

Note that our definition of stereotype does not require the words in $W$ to form distinct stereotypical groups but also considers a continuous distribution of word biases stereotypical as long as there are words closer to a certain group and other words closer to another groups. Therefore, our stereotype definition is not limited to extreme cases with fundamentally different groups but is also sensitive to individual words being stereotypical for certain protected groups.
Yet, the greater the differences or the more distinct groups appear, the higher is the stereotype considered.

In the following we will use a notation for bias score functions in general: $b(\bm{w},A)$ measuring the bias of one word. For sets of words, there are two notations reflecting two different strategies presented in the literature: Bias scores measuring bias over all neutral words jointly (MAC and Direct Bias) will be written as $b(W,A)$ and bias scores measuring the bias over two groups of neutral words $X,Y \subset W$ (WEAT), will be written as $b(X,Y,A)$. Since the bias scores from the literature have different extreme values and different values indicating no bias, we use the following notations: $b_{min}$ and $b_{max}$ are the extreme values of $b(\cdot)$ and $b_0$ is the value of $b(\cdot)$ that means $\bm{w}$ or $W$ is unbiased. Note that $b_{min}$ and $b_0$ are not necessarily equal.
This results in the following notations for WEAT, MAC and the Direct Bias:

(1) For WEAT $b_{WEAT}(\bm{w},A) = s(\bm{w},A_1, A_2)$ and $b_{WEAT}(W,A) = d(X,Y,A_1,A_2)$ with $A = \{A_1, A_2\}$ limited to two attribute sets and $X \cup Y = W$, $X \cap Y = \emptyset$ two predefined groups of the neutral words. $b_{min} = -1$, $b_{max} = 1$ and $b_0 = 0$\\
(2) For MAC $b_{MAC}(\bm{w},A) = \frac{1}{|A|} \sum_{A_j \in A} S(\bm{w}, A_j)$ (MAC score for $W = \{\bm{w}\}$ and $b_{MAC}(W,A) = MAC(W,A)$ (in the MAC notation, $T$ is corresponding to $W$).  $b_{min} = 0$, $b_{max} = 2$ and $b_0 = 1$\\
(3) For the Direct Bias $b_{DirectBias}(\bm{w},A) = |\cos(\bm{w},\bm{g})|^c $ and $b_{DirectBias}(W,A) = \frac{1}{|W|} \sum_{\bm{w} \in W} |\cos(\bm{w},\bm{g})|^c$ (in the Direct Bias notation (Equation \ref{eq:direct_bias}) $N$ is used instead of $W$). Here $\bm{g}$ refers to the bias direction, which is obtained using two attribute sets $A_1$ and $A_2$ as described in Section \ref{sec:Bolukbasi-directbias}. Hence, $A = \{A_1, A_2\}$ is limited to two attribute sets. $b_{min} = 0$, $b_{max} = 1$ and $b_0 = 0$.

\subsection{Requirements for Bias Metrics}
\label{sec:requirements}

Based on the definitions of bias explained in Section \ref{sec:formal_bias_def}, we introduce two novel properties which a bias score function should fulfill to quantify bias in a meaningful way: \emph{trustworthiness} and \emph{magnitude-comparability}. Both properties will be explained and formally defined in the following.
The goal of both properties is to ensure that biases can be quantified in a way such that bias scores can be safely compared between different embedding models and debiasing methods can be evaluated without risking to overlook bias. We acknowledge that there are also use cases where these properties can be neglected, for example the tests of \cite{weat}. Nevertheless, there plenty of statements in the literature (e.g. \cite{seat}) that suggest biases being overlooked by state-of-the-art bias scores, so it is highly relevant to determine which bias scores suffice these properties and which do not.
Furthermore, in Section \ref{sec:skew_stereo} we propose sensitivity to skew and stereotype as two additional criteria that allow a closer look into what kinds of biases certain bias scores can detect.

\subsubsection{Comparability}

The goal of our first property, \emph{magnitude-comparability}, is to ensure that bias scores are comparable between different embeddings. This is necessary to make statements about embedding models being more or less biased than others, which includes comparing debiased embeddings with their original counterparts.
A necessary condition for such comparability is the possibility to reach the extreme values of $b_{min}$ and $b_{max}$ of $b(\cdot)$ in different embedding spaces depending only on the neutral words.

More formally, assume a word embedding is fixed with embedding space E. Considering $W$, a set of words in $E$ and $A_i \in A$ sets of attributes in $E$, a bias score function $b(W,A)$ maps $W$ and the set of attribute sets $A$ to a real number. The set of words $W$ could also be divided into subsets $X$ and $Y$, corresponding to two attribute sets $A$ and $B$ (as done in WEAT). We propose the following  requirements:

\begin{definition}[\extremaVgl]
	\label{def:max_amplitutde}
	We call the bias score function $b(W,A)$ \extremaVgl\ if, for a fixed number of target words in set $W$ (including the case $W = \{\bm{w}\}$, the maximum bias score $b_{max}$ and the minimum bias score $b_{min}$ are independent of the attribute sets in A:
	\begin{align}
	\label{eq:max_min_value}
	\max_{W, |W|=const} b(W,A) = b_{max} , \quad 
	\min_{W, |W|=const} b(W,A) = b_{min} \; \forall \;  A.
	\end{align}
\end{definition}

\subsubsection{Trustworthiness}
The second novel property of \emph{trustworthiness} defines whether we can trust a bias score to report any bias in accordance to definitions \ref{def:word_bias} and \ref{def:set_bias}, i.e. the bias score can only reach $b_0$, which indicates fairness, if the observed word is equidistant to all protected groups or all words in the observed set of words are unbiased. This is important, because even if a set of words is mostly unbiased or word biases cancel out on average, individual biases can still be harmful and should thus be detected.
The requirement for the consistency of the minimal bias score $b_0$s can be formulated in a straight forward way using the similarities to the attribute sets $A_i$.

\begin{definition}[\minVgl]
	\label{def:min_max_bias}
	Let $b_0$ be the bias score of a bias score function, that is equivalent to no bias being measured.
	We call the bias score function $b(\bm{w},A)$ 
	\minVgl\ if 
	\begin{align}
	\label{eq:cond_neutral}
	b(\bm{w},A) = b_0 \iff s(\bm{w},A_i) = s(\bm{w},A_j) \; \forall \;  A_i, A_j \in A.
	\end{align}
	Analogously for scores that use at set of words $W = \{ \bm w_1, ..., \bm w_m\}$,
	we say $b(W,A)$ is \minVgl\ if
	\begin{align}
	\label{eq:cond_neutral2}
	b(W,A) = b_0 \iff s(\bm{w}_k,A_i) = s(\bm{w}_k,A_j) \; \forall \;  A_i, A_j \in A, k\in\{1,...,m\}.
	\end{align}
\end{definition}

\subsubsection{Sensitivity to Skew and Stereotype}

We formulate two additional criteria in alignment to the skew and stereotype definitions. Since the requirements introduced in \ref{sec:requirements} are critical to achieve comparable and trustworthy measurements in terms of a bias magnitude, the following definitions are simply meant to distinguish what type of biases a bias score function is sensitive towards.

\begin{definition}[\skewsensitive]
	\label{def:skew_sensitive}
	Let $b_0$ be the value of a bias score function, that is equivalent to no bias being measured.
	We call the bias score function $b(W,A)$ 
	\skewsensitive\ if 
	\begin{align}
	\label{eq:cond_skew}
	& \exists A_i, A_j \in A: mean_{\bm{w} \in W} s(\bm{w},A_i) > mean_{\bm{w} \in W} s(\bm{w},A_j) \nonumber  \\
	\Rightarrow\ & b(W,A) \neq b_0. 
	\end{align}
\end{definition}

In lay terms, skew-sensitivity reflects the fact that a bias score should increase if neutral words are (on average) more similar to one protected group than another.

\begin{definition}[\stereosensitive]
	\label{def:stereo_sensitive}
	Let $b_0$ be the value of a bias score function, that is equivalent to no bias being measured.
	We call the bias score function $b(W,A)$ 
	\stereosensitive\ if 
	\begin{align}
	\label{eq:cond_stereo}
	& \exists \bm{w_1},\bm{w_2} \in W, A_i, A_j \in A:  s(\bm{w_1},A_i)-s(\bm{w_1},A_j) \neq  s(\bm{w_2},A_i)-s(\bm{w_2},A_j) \nonumber \\
	\Rightarrow\ & b(W,A) \neq b_0.
	\end{align}
	
\end{definition}

In lay terms, this definition refers to the fact that an embedding, which places words differently with respect to its relative distance as regards to (opposing) protected groups is biased as regards the axis which is described by these two groups.

As stated in the following theorem, bias score functions that are \minVgl\ are both \skewsensitive\ and \stereosensitive. From there follows that a metric that is not \skewsensitive\ or not \stereosensitive\ cannot be \minVgl .

\begin{theorem}
\label{theor:minVgl_skew_stereo}
	A bias score function $b(W,A)$ that is \minVgl\ is also \skewsensitive\ and \stereosensitive .
\end{theorem}
\begin{proof}
\label{proof:minVgl_skew_stereo}
Let $W$ be skewed with regard to attributes in $A$, then, according to Definition \ref{def:skew_bias},
\begin{eqnarray}
\exists A_i, A_j \in A: mean_{\bm{w} \in W} s(\bm{w},A_i) > mean_{\bm{w} \in W} s(\bm{w},A_j).
\end{eqnarray}
In that case 
\begin{eqnarray}
\exists w \in W, A_i, A_j \in A:  s(\bm{w},A_i) \neq s(\bm{w},A_j).
\end{eqnarray}
Further, let $b(W,A)$ be a \minVgl\ bias score function. From Definition \ref{def:min_max_bias} directly follows that $b(W,A) \neq b_0$ given a skewed set of words $W$.

Analogously, for $W = W_1 \cup W_2$ with $W_1$ and $W_2$ reflecting different stereotypes with regard to (at least) two attributes $A_i$, $A_j \in A$ and $b(W,A)$ an \minVgl\ bias score function:
\begin{eqnarray}
& \exists \bm{w_1} \in W_1, \bm{w_2} \in W_2:  s(\bm{w_1},A_i)-s(\bm{w_1},A_j) \neq  s(\bm{w_2},A_i)-s(\bm{w_2},A_j) \\
\Rightarrow & \exists w \in \{\bm{w_1}, \bm{w_2}\}:  s(\bm{w},A_i) \neq s(\bm{w},A_j) \\
\Rightarrow & b(W,A) \neq b_0
\end{eqnarray}
\end{proof}

\subsection{Properties of bias scores from the literature}

As a major contribution of this work, we apply the novel properties for bias score functions to all cosine based bias scores from the literature. Table \ref{tab:overview} gives an overview over our findings. The detailed analyses follow in Section \ref{sec:weat_analysis} for WEAT, Section \ref{sec:analysis_mac} for MAC and Section \ref{sec:analysis_db} for the Direct Bias.

\begin{table}[t] 
\centering
\caption{Overview over the properties of bias scores.}
\label{tab:overview}
\begin{tabular}{ ccc ccc cc}
\toprule
bias score & comparable & trustworthy & bias score & comparable & trustworthy & skew & stereotype \\
\cmidrule(r){1-3} \cmidrule(l){4-8}
\weatw &  x & \checkmark & \weat &  \checkmark & x & x & (\checkmark)  \\
\macw & x & x & \mac &  x & x & x & x\\
\dbw & \checkmark & x & \db &  \checkmark & x & x & x \\
\bottomrule
\end{tabular}
\end{table}

\subsubsection{Analysis of the WEAT Score}
\label{sec:weat_analysis}
In the following, we detail properties of the WEAT score in light of the definitions stated above. First, we focus on the word-wise biases as reported by $s(\bm{w},A,B)$.

\begin{theoremrep}
\label{theor:weat_s_extrema}
	The bias score function $s(\mathbf{w},A,B)$ of WEAT is not \extremaVgl.
\end{theoremrep}
\begin{inlineproof}
For the proof see Section \ref{proof:weat_s_extrema} in the supplementary material.
\end{inlineproof}
\begin{toappendix}
\begin{proof}
\label{proof:weat_s_extrema}
With $\mathbf{\hat{a}} = \frac{1}{|A|} \sum_{\mathbf{a} \in A} \frac{\mathbf{a}}{||\mathbf{a}||}$ and $\mathbf{\hat{b}}$ analogously defined, we can rewrite
\begin{align}
    s(\mathbf{w},A,B) &=& \frac{\mathbf{w} \cdot \mathbf{\hat{a}}}{||\mathbf{w}||} - \frac{\mathbf{w} \cdot \mathbf{\hat{b}}}{||\mathbf{w}||} \\
    &=& \frac{\mathbf{w}}{||\mathbf{w}||} \cdot \Big(\mathbf{\hat{a}} - \mathbf{\hat{b}} \Big) \\
    &=& cos(\mathbf{w},\mathbf{\hat{a}} - \mathbf{\hat{b}}) ||\mathbf{\hat{a}} - \mathbf{\hat{b}}||.
\end{align}
Hence we can show that the extrema depend on the attribute sets $A$ and $B$:
\begin{align}
    max_{\mathbf{w}} s(\mathbf{w},A,B) = ||\mathbf{\hat{a}} - \mathbf{\hat{b}}||, \\
    min_{\mathbf{w}} s(\mathbf{w},A,B) = -||\mathbf{\hat{a}} - \mathbf{\hat{b}}||
\end{align}
The statement follows.
\end{proof}
\end{toappendix}

\begin{figure}[tb]
	\centering
	\includegraphics[scale=0.2]{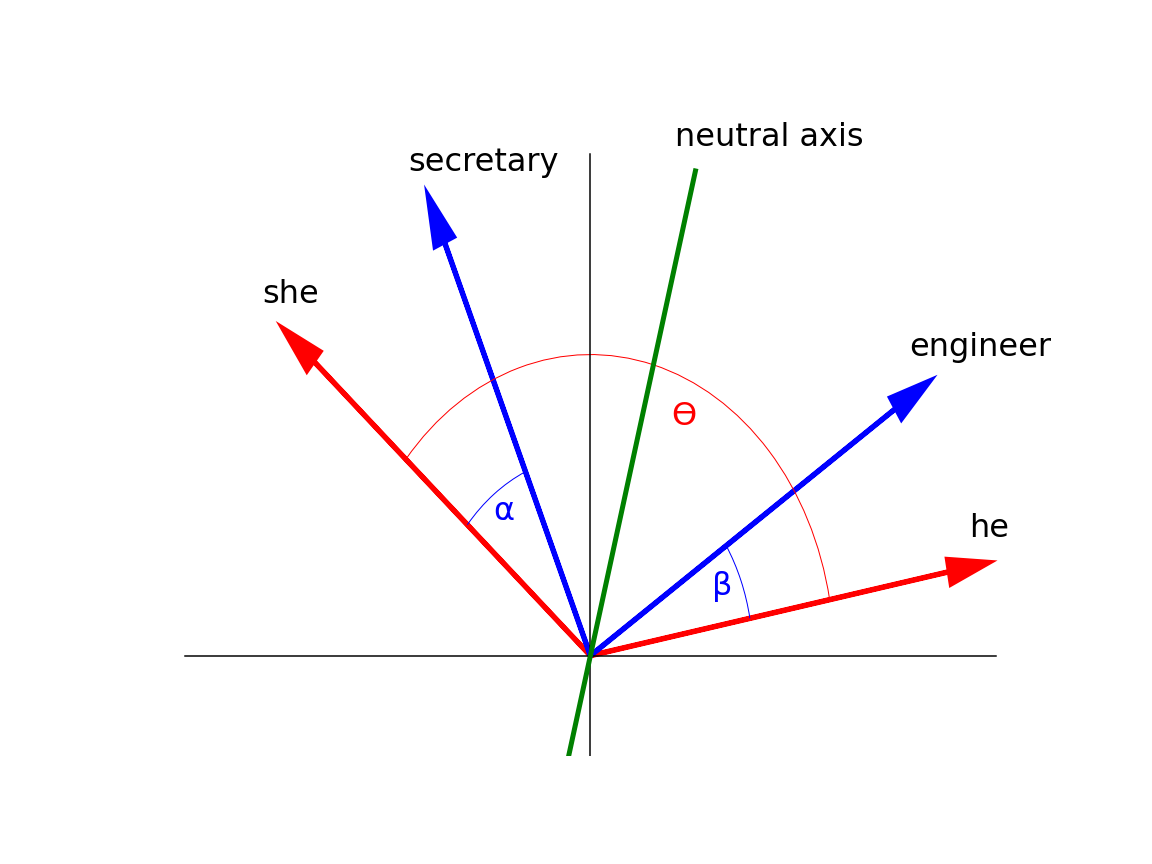}
	\caption{Simplified visualization of gender stereotypes in embeddings. The occupation vectors are biased if $0 < \alpha < \frac{\theta}{2}$ or $0 < \beta < \frac{\theta}{2}$.}
	\label{fig:bias_example_angles}
\end{figure}

\begin{theoremrep}
\label{theor:weat_s_min}
	The bias score function $s(\bm{w},A,B)$ of WEAT is \minVgl.
\end{theoremrep}
\begin{inlineproof}
For the proof see Appendix \ref{proof:weat_s_min}.
\end{inlineproof}
\begin{toappendix}
\begin{proof}
\label{proof:weat_s_min}
This follows directly from the definition of $s(\bm{w},A,B)$ (equation \eqref{eq:weat_attr_sim}):
\begin{align}
& s(\bm{w},A,B) = s(\bm{w},A) - s(\bm{w},B) = 0 \\
\iff & s(\bm{w},A) = s(\bm{w},B)
\end{align}
\end{proof}
\end{toappendix}

Next, we focus on the properties of the effect size $d(X,Y,A,B)$. Note that the effect size is not specified for cases, where  $s(\bm{w},A,B) = s(\bm{w'},A, B) \forall \bm{w}, \bm{w'} \in X \cup Y$ due to its denominator. This is highly problematic considering Definition \ref{def:min_max_bias}, which states that a bias score should be $0$ in that specific case (which implies perfect fairness).
Furthermore, the following theorem shows that WEAT can report no bias even if the embeddings contain associations with the bias attributes.

\begin{theorem}
\label{theor:weat_d_min}
	The effect size $d(X,Y,A,B)$ of WEAT is not \minVgl.
\end{theorem}
\begin{proof}
For the WEAT score $b_0 = 0$. 


For four words $\bm{w}_1, \bm{w}_2, \bm{w}_3, \bm{w}_4$ and $s(\bm{w}_1,A,B) = s(\bm{w}_3,A,B)$ and $s(\bm{w}_2,A,B) = s(\bm{w}_4,A,B)$, the effect size
\begin{align}
\label{eq:eff_size}
& d(\{\bm{w}_1,\bm{w}_2\},\{\bm{w}_3,\bm{w}_4\},A,B) = \\ 
& \frac{(s(\bm{w}_1,A,B) + s(\bm{w}_2,A,B)) - (s(\bm{w}_3,A,B) + s(\bm{w}_4,A,B))}{2 \cdot \mathrm{stddev}_{\bm{w} \in \{\bm{w}_1,\bm{w}_2,\bm{w}_3, \bm{w}_4\}} s(\bm{w},A,B)}
\end{align}
is $0$, if $s(\bm{w}_1,A,B) \neq s(\bm{w}_2,A,B)$ (otherwise $d$ is not defined). 
Now, for the simple case $A=\{\bm{a}\}, B=\{\bm{b}\}$ and assuming all vectors having length $1$, we see
\begin{align}
&s(\bm{w}_1,A,B) = s(\bm{w}_3,A,B) \nonumber \\
\iff & \bm{a}\cdot\bm{w}_1 - \bm{b}\cdot\bm{w}_1 = \bm{a}\cdot\bm{w}_3 - \bm{b}\cdot\bm{w}_3 \nonumber \\
\iff & \bm{a}\cdot(\bm{w}_1 - \bm{w}_3) - \bm{b}\cdot(\bm{w}_1 - \bm{w}_3) = 0 \nonumber \\
\iff & (\bm{a} - \bm{b})\cdot(\bm{w}_1 - \bm{w}_3) = 0.
\end{align}
This implies that, if the two vectors $\bm{a} - \bm{b}$ and $\bm{w}_1 - \bm{w}_3$ are orthogonal (and e.g.\ $s(\bm{w}_2,A,B)=0$), the WEAT score returns 0. In this case, there exist $\bm{a}, \bm{b}, \bm{w}_1, \bm{w}_3$ with $s(\bm{w}_1,A,B) = s(\bm{w}_3,A,B) \neq 0$ and accordingly $s(\bm{w}_1,A) \neq s(\bm{w}_1,B)$.
See Figure \ref{fig:bias_example_weat} for an example, with $\bm{a} =$ he, $\bm{b} =$ she, $\bm{w}_1 =$ engineer, $\bm{w}_3 =$ secretary.
\end{proof}

\begin{figure}[tb]
	\centering
	\includegraphics[scale=0.2]{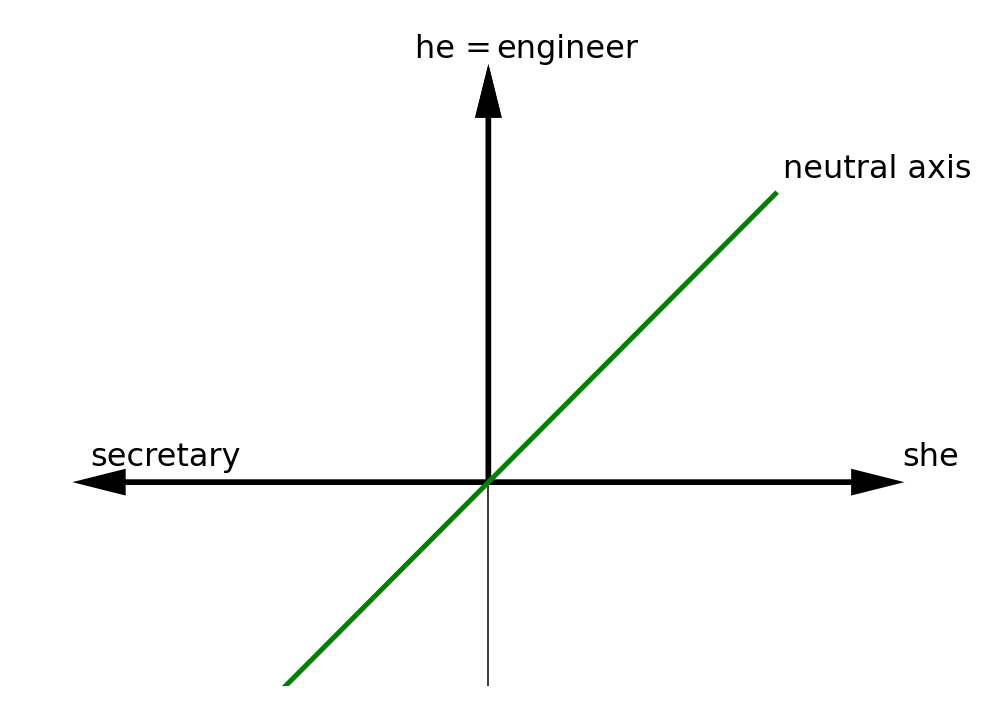}
	\caption{This illustrates a case, where WEAT would report no bias since secretary and engineer have the same relative similarity towards he and she.}
	\label{fig:bias_example_weat}
\end{figure}

In order to show that the effect size of WEAT is \extremaVgl, we need the following lemma.

\begin{lemmarep}\label{theor:schrankeAllg}
Let $x_1,...,x_n \in \R$ be real numbers. Let $\hat{\mu},\hat{\sigma}$ denote the empirical estimate of mean and standard deviation of the $x_i$. Then, for any selection of indices $i_1,...,i_m$, with $i_j \neq i_k$ for $j \neq k$, the following bound holds
\begin{align*}
    \left| \sum_{j = 1}^m \frac{x_{i_j} - \hat{\mu}}{\hat{\sigma}} \right| \leq \sqrt{m \cdot (n-m)}.
\end{align*}
Furthermore, for $0 < m < n$ the bound is obtained if and only if all selected resp. non-selected $x_i$ have the same value, i.e. $x_{i_j} = \hat{\mu} + s\sqrt{\frac{n-m}{m}}\hat{\sigma}$ and all other $x_k = \hat{\mu} -s\sqrt{\frac{m}{n-m}}\hat{\sigma}$ with $s \in \{-1,1\}$.
\end{lemmarep}
\begin{inlineproof}
For the proof see Appendix \ref{proof:schrankeAllg}.
\end{inlineproof}
\begin{toappendix}
\begin{proof}
\label{proof:schrankeAllg}
For cases $m = 0$ or $m = n$ the statement is trivial. So assume $0 < m < n$.
\newcommand{\x}{\bm{x}}
\newcommand{\s}{\bm{s}}
Let $f(x) = ax+b$ be an affine function. Then the images of $x_i$ under $f$ have mean $a\hat{\mu}+b$ and standard deviation $|a|\hat{\sigma}$. On the other hand, we have
\begin{align*}
    \frac{f(x_i)-(a\hat{\mu}+b)}{|a|\hat{\sigma}}
    &= \frac{(ax_i+b)-(a\hat{\mu}+b)}{|a|\hat{\sigma}}
    = \text{sgn}(a)\frac{x_i-\hat{\mu}}{\hat{\sigma}}.
\end{align*}
Thus, applying $f$ does not change the bound and therefore we may reduce to case of $\hat{\mu} = 0$ and $\hat{\sigma} = 1$. This allows us to rephrase the problem of finding the maximal bound as an quadratic optimization problem:
\begin{align*}
    \min \quad& \s^\top \x \\
    \text{s.t.} \quad & \x^\top \x = n\\
                      & \bm{1}^\top \x = 0,
\end{align*}
where $\s = (1,...,1,0,...,0)^\top$, $\x = (x_1,...,x_n)^\top$ and $\bm{1}$ denotes the vector consisting of ones only. Notice, that we assumed w.l.o.g. that $i_1,...,i_m = 1,...,m$. Furthermore, we made use of the symmetry properties to replace $\max |\s^\top \x|$ by the minimizing statement above, $\hat{\mu} = 0$ is expressed by the last and $\hat{\sigma} = 1$ by the first constrained (recall that $\hat{\sigma} =  \sqrt{1/n \x^\top \x-\hat{\mu}^2}$).
Notice, that $\nabla_x \x^\top\x-n = 2\x$ and $\nabla_x \bm{1}^\top \x = \bm{1}$ are linear dependent if and only if $\x = a \bm{1}$ for some $a \in \R$, thus, as $0 = a \bm{1}^\top \bm{1} = an$ if and only if $a = 0$ and $(0\bm{1})^\top (0\bm{1}) = 0$, there is no feasible $\x$ for which the KKT-conditions do not hold and we may therefore use them to determine all the optimal points.

The Lagrangien of the problem above and its first two derivatives are given by
\begin{align*}
    L(\x,\lambda_1,\lambda_2) &= \s^\top \x - \lambda_1 (\x^\top I \x - n) - \lambda_2 \bm{1}^\top \x \\
    \nabla_x L(\x,\lambda_1,\lambda_2) &= \s - 2 \lambda_1 \x - \lambda_2 \bm{1} \\
    \nabla^2_{x,x} L(\x,\lambda_1,\lambda_2) &= - 2 \lambda_1 I.
\end{align*}
We can write $\nabla_x L(\x,\lambda_1,\lambda_2) = 0$ as the following linear equation system:
\begin{align*}
    \begin{bmatrix} 2x_1 & 1 \\ 2x_2 & 1 \\ \vdots & \vdots \\ 2x_n & 1 \end{bmatrix} \begin{bmatrix} \lambda_1 \\ \lambda_2 \end{bmatrix} &= \underbrace{\begin{bmatrix} 1 \\ 1 \\ \vdots \\ 0 \end{bmatrix}}_{=\s}.
\end{align*}
Subtracting the first row from row $2,...,m$ and row $m+1$ from row $m+2,...,n$ we see that $2(x_k-x_1)\lambda_1 = 0$ for $k=1,...,m$ and $2(x_k-x_{m+1})\lambda_1 = 0$ for $k = m+2,...,n$, which either implies $\lambda_1 = 0$ or $x_1=x_2 = ... = x_m$ and $x_{m+1}=x_{m+2} = ... = x_n$. However, assuming $\lambda_1 = 0$ would imply that $\lambda_2 = 1$ from the first row and $\lambda_2 = 0$ from the $m+1$th row, which is a contradiction. Thus, we have  $x_1=x_2 = ... = x_m$ and $x_{m+1}=x_{m+2} = ... = x_n$. But the second constraint from the optimization problem can then only be fulfilled if $mx_1+(n-m)x_{m+1} = 0$ and this implies $x_{m+1} = -\frac{m}{n-m}x_1$. In this case the first constraint is equal to $n = m x_1^2 + (n-m) \left(\frac{m}{n-m} x_1\right)^2$
, which has the solution $x_1 = \pm\sqrt{\frac{n-m}{m}}$.

Set $\x^* = (-\sqrt{\frac{n-m}{m}},...,-\sqrt{\frac{n-m}{m}},\sqrt{\frac{m}{n-m}},...,\sqrt{\frac{m}{n-m}})$. Then $\x^*$ and $-\x^*$ are the only possible KKT points as we have just seen. Plugging $\x^*$ into the equation system above and solving for $\lambda_{1/2}$ we obtain
\begin{align*}
    \lambda_1^* &= -\frac{1}{2 \left(\sqrt{\frac{m}{n - m}} + \sqrt{\frac{n - m}{m}}\right)} \\
    \lambda_2^* &= \frac{\sqrt{\frac{m}{n - m}}}{\sqrt{\frac{m}{n - m}} + \sqrt{\frac{n-m}{m}}}
\end{align*}
Now, as $\nabla^2_{x,x} L(\x^*,\lambda_1^*,\lambda_2^*) = \left(\sqrt{\frac{m}{n - m}} + \sqrt{\frac{n - m}{m}}\right)^{-1} I$ is positive definite, we see that $\x^*$ is a global optimum, indeed. The statement follows.
\end{proof}
\end{toappendix}

\begin{theoremrep}
\label{theor:weat_extrema}
The effect size $d(X,Y,A,B)$ of WEAT with $X=\{\bm{x}_1, \ldots,\bm{x}_m\}, Y=\{\bm{y}_1, \ldots,\bm{y}_m\}$ is \extremaVgl.
\end{theoremrep}
\begin{inlineproof}
For the proof see Appendix \ref{proof:weat_extrema}.
\end{inlineproof}
\begin{toappendix}
\begin{proof}
\label{proof:weat_extrema}
With $c_i = s(\bm{x_i},A,B)$, $c_{i+m} = s(\bm{y_i},A,B)$, $n=2m$, $\hat{\mu}=1/n \sum^{n}_{i=1} c_i$ and $\hat{\sigma}=\sqrt{1/n\sum^{n}_{i=1} (c_i - \mu)^2}$, we have
\begin{align}
    d &= \frac{1/m \sum^m_{i=1} c_i - 1/m \sum^{2m}_{i=m+1} c_i}{\hat{\sigma}} \\
      &= \frac{\sum^m_{i=1} c_i - \sum^{2m}_{i=m+1} c_i + \sum^m_{i=1} c_i - \sum^m_{i=1} c_i}{m\hat{\sigma}} \nonumber \\
      &= \frac{2\sum^m_{i=1} c_i - 2m\hat{\mu}}{m\hat{\sigma}} \nonumber \\
      &= \frac{2}{m} \sum^m_{i=1}\frac{c_i - \hat{\mu}}{\hat{\sigma}} \in[-2,2]
\end{align}
where the last statement follows from Lemma \ref{theor:schrankeAllg} with $\sum^m_{i=1}\frac{c_i - \hat{\mu}}{\hat{\sigma}} \in [-m,m]$. The extreme value $\pm 2$ is reached if $c_1=\ldots =c_m=-c_{m+1}=\ldots =-c_{2m}$, which can be obtained by setting $\bm{x}_1=\ldots =\bm{x}_m=-\bm{y}_1=\ldots =-\bm{y}_m$, independently of $A$ and $B$ as long as $A \neq B$ and $\sum_{\bm{a}_i \in A} \bm{a}_i/\|\bm{a}_i\| \neq 0 \neq \sum_{\bm{b}_i \in B} \bm{b}_i/\|\bm{b}_i\|$.
\end{proof}
\end{toappendix}

The proof of Theorem \ref{theor:weat_extrema} shows that the effect size reaches its extreme values only if all $x \in X$ achieve the same similarity score $s(\bm{x},A,B)$ and $s(\bm{y},A,B) = -s(\bm{x},A,B) \; \forall \; y \in Y$, i.e. the smaller the variance of $s(\bm{x},A,B)$ and $s(\bm{y},A,B)$ the higher the effect size. For one thing, this limits WEAT's predictability to the partitioning of neutral words into $X$ and $Y$. On the other hand, the differences $mean_{\bm{x} \in X} s(\bm{x},A,B) - mean_{\bm{y} \in Y} s(\bm{y},A,B)$ does not affect the effect size values. The dependence on the variance explains the behavior shown in Section \ref{sec:req_motivation}. This shows that we cannot take low effect sizes as a guarantee for low biases in the embeddings.

Finally, we probe the effect size for skew and stereotype sensitivity:
\begin{theoremrep}
	\label{theor:weat_d_stereo}
	The effect size $d(X,Y,A,B)$ of WEAT is \stereosensitive\ assuming that the groups $X$ and $Y$ correctly classify for the direction of stereotype, i.e. all words in $X$ are closer to the attribute set $A$ (compared to $B$) than words in $Y$:
	\begin{align}
	\label{eq:stereo_assumption}
	    s(\bm{x},A)-s(\bm{x},B) > s(\bm{y},A)-s(\bm{y},B) \; \forall \bm{x} \in X, \bm{y} \in Y,
	\end{align}
	or vice-versa:
	\begin{align}
	    s(\bm{x},A)-s(\bm{x},B) < s(\bm{y},A)-s(\bm{y},B) \; \forall \bm{x} \in X, \bm{y} \in Y
	\end{align}
\end{theoremrep}
\begin{inlineproof}
For the proof see Appendix \ref{proof:weat_d_stereo}.
\end{inlineproof}
\begin{toappendix}
\begin{proof}
\label{proof:weat_d_stereo}
We transform the effect size (Equation \eqref{eq:weat_eff_size}) into
\begin{align}
    d(X,Y,A,B) = \frac{mean_{(\bm{x}, \bm{y}) \in (X,Y)} s(\bm{x},A,B) - s(\bm{y},A,B)}{stddev_{\bm{w} \in X \cup Y} s(w,A,B)}.
\end{align}
From our assumption follows that $|d(X,Y,A,B)| \neq 0$.
\end{proof}
\end{toappendix}

\begin{theoremrep}
\label{theor:weat_d_skew}
	The bias score function $d(X,Y,A,B)$ of WEAT is not \skewsensitive.
\end{theoremrep}
\begin{inlineproof}
For the proof see Appendix \ref{proof:weat_d_skew}.
\end{inlineproof}
\begin{toappendix}
\begin{proof}
\label{proof:weat_d_skew}
Let $X \cup Y = \{\bm x,\bm y\}$ be a skewed set with $s(\bm{w},A) > s(\bm{w},B) \forall \bm{w} \in X \cup Y$. Assuming that $s(\bm x, A, B) = s(\bm y, A, B) \neq 0$, which depicts a skew but no stereotype, from the definition of the effect size follows:
\begin{align}
    d(X,Y,A,B) = 0,
\end{align}
which contradicts equation \eqref{eq:cond_skew} with $b_0 = 0$ for WEAT.
\end{proof}
\end{toappendix}

To conclude this theorems, we found that the word-wise bias reported by WEAT is in fact \minVgl, but not \extremaVgl. The effect size, which is applied to measure bias of two sets of target words, is \extremaVgl, but not \minVgl. This is due to the fact that the effect size is sensitive to stereotypical differences between the sets $X$ and $Y$ (which is exactly the intuition behind WEAT), but it is not \skewsensitive. Furthermore the magnitude of the effect size depends on the inner-group variance of $s(\bm{w},A,B)$ for groups $X$ and $Y$, but not one the difference between those groups, which should be taken into account when interpreting it.\\
This shows that WEAT can only be used to quantify stereotypical biases as anticipated by the groups $X$ and $Y$, while skew and stereotypes that contradict the partitioning will not be detected. And most importantly, while high effect sizes are a certain indicator of bias, low effect sizes are not that meaningful. Hence, when it comes to quantitatively comparing biases between different embeddings or evaluating debiasing methods, we argue that WEAT is not sufficient. At most, it could be used as an additional measure for specific stereotypes using, for instance, the tests from \cite{weat}. Of course, if ones goal is to only prove the presence of bias instead of quantifying it, WEAT is still useful.\\

\subsubsection{Analysis of the MAC Score}
\label{sec:analysis_mac}

On contrary to the bias definitions summarized in Section \ref{sec:bias_definition}, the MAC score follows a slightly different intuition. Instead of contrasting over the association of words towards different attributes (as WEAT does) or measuring the correlation with a bias direction (Direct Bias), it depicts the average correlation of neutral words and bias attributes. In the following we show several situations where the metric thereby contradicts our bias definition. This includes detecting bias on fair embeddings and measuring no bias in biased embeddings.\\


\begin{theoremrep}
\label{theor:mac_min}
The bias score function $MAC(T,A)$ of MAC is not \minVgl.
\end{theoremrep}
\begin{inlineproof}
For the proof see Appendix \ref{proof:mac_min}.
\end{inlineproof}
\begin{toappendix}
\begin{proof}
\label{proof:mac_min}
An ideal MAC score, indicating fairness, would be $b_0 = \pm 1$.
We again consider the example case depicted in Figure \ref{fig:bias_example_angles} with $\bm{w} = \bm{secretary},  A_1 = \{\bm{she}\}, A_2 = \{\bm{he}\}$.

First, consider $\theta = \pi$ and $0 < \alpha < \frac{\theta}{2}$. In that case, $\bm{secretary}$ would be closer to $\bm{she}$ than $\bm{he}$, which reflects a biased representation and does not satisfy Eq.\ \eqref{eq:cond_neutral}. Yet, if we apply the MAC score $MAC(T,A)$ with $T = \{\bm{w}\}, A=\{A_1,A_2\}$, due to the symmetry of the cosine function we get
\begin{align}
MAC(T,A) &= \frac{1}{2} ((1-cos(\alpha)) + (1-cos(\pi-\alpha))) \\
&= 1,
\end{align}
which implies the word secretary to be perfectly fair embedded.\\

As a second example, if we considered $\alpha = \frac{\theta}{2}, \theta \neq \pi$, and thus $\bm{secretary}$ being equidistant to $A_1$ and $A_2$. Yet, we would get
\begin{align}
MAC(T,A) &= \frac{1}{2} ((1 - cos(\alpha)) + (1 - cos(\theta - \alpha))) \\
&= 1 - cos\left(\frac{\theta}{2}\right).
\end{align}
Now, due to $\theta \neq \pi$, we still measure a bias, although based on eq.\ \eqref{eq:cond_neutral} the occupation is unbiased towards both gender words. Also, the larger $|\theta-\pi|$ the more extreme is the bias according to the MAC score.
\end{proof}
\end{toappendix}

\begin{theoremrep}
The bias score functions \mac\ and \macw\ are not \extremaVgl.
\end{theoremrep}
\begin{inlineproof}
For the proof see Section \ref{proof:mac_extrema} in the supplementary material.
\end{inlineproof}
\begin{toappendix}
\begin{proof}
\label{proof:mac_extrema}
With $\mathbf{\hat{a_i}} = \frac{1}{|A_i|} \sum_{\mathbf{a_i} in A_i} \frac{\mathbf{a_i}}{||\mathbf{a_i}||}$, $|A_i| = |A_j| \; \forall A_i, A_j \in A$ and $ W = \{\mathbf{w}\}$ we can rewrite MAC:
\begin{align}
    MAC(W,A) &=& \frac{1}{|A|} \sum_{A_i \in A} \frac{\mathbf{w} \cdot \mathbf{\hat{a_i}}}{||\mathbf{w}||}\\
    &=& \frac{\mathbf{w}}{||\mathbf{w}||} \cdot \frac{1}{|A|} \sum_{A_i \in A}  \mathbf{\hat{a_i}} \\
    &=& cos(\mathbf{w},\mathbf{\hat{a}}) ||\mathbf{\hat{a}}||,
\end{align}
with $\mathbf{\hat{a}} = \frac{1}{|A|} \sum_{A_i \in A}  \mathbf{\hat{a_i}}$.
Hence we can show that the extrema depend on the attribute sets $A$ and $B$:
\begin{align}
    max_{W} MAC(W,A) = ||\mathbf{\hat{a}}||, \\
    min_{W} MAC(W,A) = -||\mathbf{\hat{a}}||
\end{align}
Choosing $|W| > 1$ does not change the extrema. The statement follows.
\end{proof}
\end{toappendix}

\begin{theoremrep}
\label{theor:mac_stereo_skew}
The bias score function $MAC(T,A)$ of MAC is neither \stereosensitive\ nor \skewsensitive.
\end{theoremrep}
\begin{inlineproof}
For the proof see Appendix \ref{proof:mac_stereo_skew}.
\end{inlineproof}
\begin{toappendix}
\begin{proof}
\label{proof:mac_stereo_skew}
We showed in the proof of Theorem \ref{theor:mac_min} that if the angle between both attributes is $\pi$, the average bias of an individual word will be $b_0 = 1$.\\
Now consider a purely stereotypical scenario with $\bm{secretary}$ and $\bm{engineer}$ as depicted in Figure \ref{fig:bias_example_angles} with $\theta = \pi$, and $\alpha = \beta$. With both individual words' MAC scores being $1$, the MAC of both words will also be $1$.\\ 
On the other hand, a purely skewed representation with $\bm{engineer} = \bm{secretary}$ would also result in a MAC score of $1$.
\end{proof}
\end{toappendix}

Considering these theorems, we strongly argue against using the MAC score at all, since it is neither reliable to measure skew nor stereotype. Furthermore, it might falsely detect a bias simply because the embedding vectors are represented closely in general.

\subsubsection{Analysis of the Direct Bias}
\label{sec:analysis_db}
\begin{theoremrep}
\label{theor:direct_bias_extrema}
The Direct Bias is \extremaVgl\ for $c \geq 0$.
\end{theoremrep}
\begin{inlineproof}
For the proof see Appendix \ref{proof:direct_bias_extrema}.
\end{inlineproof}
\begin{toappendix}
\begin{proof}
\label{proof:direct_bias_extrema}
For $c \geq 0$ the bias of one word $|\cos(\bm{w},\bm{g})|^c$ is in $[0,1]$. Calculating the mean over all words in $W$ does not change this bound.
\end{proof}
\end{toappendix}

\begin{theorem}
\label{theor:direct_bias_minvgl}
The Direct Bias is not \minVgl.
\end{theorem}
\begin{proof}
For the Direct Bias $b_0 = 0$ indicates no bias.
Consider a setup with two attribute sets $A=\{\bm{a_1}, \bm{a_2}\}$ and $C=\{\bm{c_1}, \bm{c_2}\}$ (we cannot call our attribute set $B$ as usual because this could be confused with the bias subspace $B$ or bias direction $b$ from the Direct Bias).\\
Using the notation from Section \ref{sec:Bolukbasi-directbias} this gives us two defining sets $D_1 = \{\bm{a_1}, \bm{c_1}\}$ $D_2 = \{\bm{a_2}, \bm{c_2}\}$. Let $\bm{a_{1}} = (-x, rx)^T = -\bm{c_{1}}, \bm{a_{2}}  = (-x, -rx)^T = -\bm{c_{2}}$ and $r > 1$.\\
The bias direction is obtained by computing the first principal component over all $(\bm{a_{i}} - \bm \mu_i)$ and $(\bm{c_{i}} - \bm \mu_i)$ with $\bm \mu_i = \frac{\bm{a_i}+\bm{c_i}}{2} = 0$. Due to $r > 1$, $\bm{b} = (0, 1)^T$ is a valid solution for the 1st principal component as it maximizes the variance
\begin{align}
    \bm{b} = argmax_{\|\bm{v}\| = 1} \sum_i (\bm{v} \cdot \bm{a_{i}})^2 + (\bm{v} \cdot \bm{c_{i}})^2.
\end{align}

According to the definition in Section \ref{sec:bias_definition}, any word $\bm{w} = (0, w_y)^T$ would be considered neutral to groups $A$ and $C$ with $s(\bm{w}, A) = s(\bm{w}, C)$ and being equidistant to each word pair $\{a_i, c_i\}$.\\
But with the bias direction $\bm{b} = (0, 1)^T$ the Direct Bias would report a maximal bias of $1$ instead of $b_0 = 0$, which contradicts Definition \ref{def:min_max_bias}.\\
On the other hand, we would consider a word  $\bm{w} = (w_x, 0)^T$ maximally biased, but the Direct Bias would report no bias.
\end{proof}

Note that the example shown in Theorem \ref{theor:direct_bias_minvgl} is an extreme case. Yet, it shows that a bias direction obtained by the PCA does not necessarily represent individual bias directions appropriately. This can lead to both over- and underestimation bias. If the variance inside the attribute groups is higher than the differences between the attributes, the bias direction (or subspace) is likely to be misleading.
(Note that this critique of the Direct Bias does not affect the debiasing algorithm of \cite{bolukbasi}, due to the equalize step.)

\section{Proposed Metric}
\label{sec:proposed_metric}

To address the shortcomings of the existing bias scores, we propose a new score based on the before mentioned bias definition: Scoring Association Means of Word Embeddings (SAME). It has a similar intuition to WEAT, in terms of measuring polarity between attribute sets, but fulfills the criteria set up in the last chapter.
In a first step we detail SAME for a setting with binary attributes and then extend it to the multiple attributes case.

\subsection{The binary case}
\label{sec:binary_attr_bias}

We use one set of target words $W$ and measure the association with two attribute sets $A_i$ and $A_j$. Furthermore, we assume that each attribute vector $\bm{a_i} \in A_i$ is normalized to unit length, so that $cos(\bm{w},\bm{a_i}) = \frac{\bm{w}}{|\bm{w}|} \cdot \bm{a_i}$. Then, we can write the similarity of a word $\bm{w}$ towards attribute set $A_i$ as
\begin{eqnarray}
\label{eq:sim_center}
s(\bm{w},A_i) &=& \frac{1}{|A_i|} \sum_{\bm{a_i} \in A_i} \frac{\bm{w}}{\|\bm{w}\|} \cdot \bm{a_i} \\
&=& \frac{\bm{w}}{\|\bm{w}\|} \cdot \left(\frac{1}{|A_i|} \sum_{\bm{a_i} \in A_i} \bm{a_i}\right) \\
&=& \frac{\bm{w}}{\|\bm{w}\|} \cdot \hat{\bm{a_i}} 
\end{eqnarray}
where $\hat{\bm{a_i}} = \left(\frac{1}{|A_i|} \sum_{\bm{a_i} \in A_i} \bm{a_i}\right)$.

Similar to WEAT, we define a pairwise bias comparing the similarity of a word $w$ towards two attribute sets $A_i$ and $A_j$ with $\hat{\bm{a_i}} \neq \hat{\bm{a_j}}$:
\begin{eqnarray}
b(\bm{w}, A_i, A_j) &=& \frac{s(\bm{w}, A_i) - s(\bm{w}, A_j)}{\|\hat{\bm{a_i}}-\hat{\bm{a_j}}\|} \\
&=& \frac{\frac{\bm{w}}{\|\bm{w}\|} \cdot \hat{\bm{a_i}} - \frac{\bm{w}}{\|\bm{w}\|} \cdot \hat{\bm{a_j}}}{\|\hat{\bm{a_i}}-\hat{\bm{a_j}}\|} \\
&=& \frac{\bm{w} \cdot (\hat{\bm{a_i}} - \hat{\bm{a_j}})}{\|\bm{w}\| \; \|\hat{\bm{a_i}}-\hat{\bm{a_j}}\|} \\
&=& cos(\bm{w},\hat{\bm{a_i}} - \hat{\bm{a_j}}).
\label{eq:pair_bias}
\end{eqnarray}
Contrary to equation \eqref{eq:weat_attr_sim} of WEAT we normalize the term, resulting in bias scores in $[-1,1]$ independent of the attributes. By transforming the equation, we can show that it has a similar notion to the Direct Bias from \cite{bolukbasi}, measuring the correlation of words with a bias direction between two attributes. The only difference is that we obtain the pairwise bias direction by averaging over individual directions instead of using the PCA.\\

The term $b(\bm{w}, A_i, A_j)$ could be contrasted over pairs $\bm{w}_1\in W_1$, $\bm{w}_2\in W_2$, similarly as done in the WEAT score, but this would results in the potential problems stated in Theorem \ref{theor:weat_d_min} / Figure \ref{fig:bias_example_weat}.
Hence, to achieve a trustworthy and comparable metric for binary bias attributes, we propose to take the mean absolute values of word-wise biases. This results in the following bias score for a set of target words $W$:
\begin{eqnarray}
\label{eq:same_set_bias}
b(W,A_i, A_j) &=& \frac{1}{|W|} \sum_{\bm{w} \in W} |b(\bm{w}, A_i, A_j)|.
\end{eqnarray}

\subsection{The multiple attributes case}
\label{sec:multi_attr_bias}
Now we extent SAME to cases with multiple attributes. Let $A = \{A_0, ..., A_n\}$ with $n \geq 1$ contain at least 2 attribute sets. To measure the bias with respect to all attributes in $A$, we construct a $n$-dimensional bias subspace from binary bias directions $\mathbf{\hat{a_i}}-\mathbf{\hat{a_0}}$ with $i \in \{1...n\}$ and $\mathbf{\hat{a_i}}$ the mean of attributes in $A_i$. Thereby, we assume that $\mathbf{\hat{a_i}} \neq \mathbf{\hat{a_j}} \; \forall i, j  \in \{1...n\}, i \neq j$.\\

Let $B$ be the bias subspace, defined by an orthonormal basis $\{\mathbf{b_1}, ..., \mathbf{b_n}\}$. The first basis vector $\mathbf{b_1}$ is obtained from the first binary bias direction, i.e.
\begin{align}
    \mathbf{b_1} = \frac{\mathbf{\hat{a_1}} - \mathbf{\hat{a_0}}}{||\mathbf{\hat{a_1}} - \mathbf{\hat{a_0}}||}.
\end{align}
The other basis vectors are obtained from the successive binary bias directions, after removing linear correlations with previous basis vectors, which ensures orthogonality
\begin{align}
    \mathbf{b_i'} &=& (\mathbf{\hat{a_i}} - \mathbf{\hat{a_0}}) - \langle \mathbf{\hat{a_i}} - \mathbf{\hat{a_0}} , \mathbf{b_{i-1}} \rangle \mathbf{b_{i-1}} - ... -  \langle \mathbf{\hat{a_i}} - \mathbf{\hat{a_0}} , \mathbf{b_0} \rangle \mathbf{b_0} \\
    \mathbf{b_i} &=& \frac{\mathbf{b_i'}}{||\mathbf{b_i'}||}
\end{align}
Thereby, we assume that $||\mathbf{b_i'}|| \geq 0$. In the case of linear dependency of the binary bias directions, i.e. one calculated $\mathbf{b_i'}$ is a zero vector, this vector can be left out, resulting in a smaller bias space $B$.

Given this subspace, the bias of a word vector $\mathbf{w}$ is described by the cosine similarities with the basis vectors of $B$
\begin{align}
    \mathbf{w_B} = \Big( \cos(\mathbf{w} , \mathbf{b_1}), ..., \cos(\mathbf{w} , \mathbf{b_n}) \Big) ^T
\end{align}
and the bias magnitude is
\begin{align}
    SAME_{word}(\mathbf{w}) := b(\mathbf{w}, A) = ||\mathbf{w_B}||
\end{align}

Consequently the overall bias of all target words in $W$ is
\begin{eqnarray}
\label{eq:set_bias_multi}
SAME(W) := b(W,A) &=& \frac{1}{|W|} \sum_{\mathbf{w} \in W} b(\mathbf{w},A).
\end{eqnarray}

A major benefit of using a vector $\mathbf{w_B}$ as an intermediate step to obtain a bias magnitude over multiple attributes is that we achieve an interpretable representation of $\mathbf{w}$ in terms of the protected groups, i.e.\ each element in $\mathbf{w_B}$ represents the association of $\mathbf{w}$ with the two groups, whose attributes were used to obtain the respective basis vector. Noteworthy, all elements represent bias in comparison to the same default group that is represented by $A_0$ and users must be aware that if bias directions correlate with on each other, their share will be only accounted for by the first basis vector. Therefor, analysing how the different bias directions correlate with each other is highly recommended to allow more sophisticated insights into how biases manifest in embedding spaces.

\subsection{Analysis of SAME}

The following theorems and their proofs detail properties of SAME in light of the above stated definitions and show that \same\ and \samew\ are \minVgl\ and \extremaVgl. Hence, we can state that SAME is a reliable bias score to quantify bias in embeddings and it can be compared between different embedding models

\begin{theoremrep}
\label{theor:own_min}
The bias score function $b(\mathbf{w},A)$ and $b(W,A)$ are \minVgl .
\end{theoremrep}
\begin{inlineproof}
For the proof see Section \ref{proof:own_min} in the supplementary material.
\end{inlineproof}
\begin{toappendix}
\begin{proof}
\label{proof:own_min}
The bias score indicating no bias is $b_0 = 0$. First, we can state that
\begin{align}
b(\mathbf{w},A_i, A_j) = 0 \iff s(\mathbf{w},A_i) = s(\mathbf{w},A_j),
\end{align}
which directly follows from the definition.
Hence 
\begin{align}
    b(\mathbf{w},A) &=& ||\mathbf{w_B}|| = ||\big( b(\mathbf{w},A_0,A_1), ..., b(\mathbf{w},A_0,A_n) \big)^T|| = 0 \\
    &\iff& s(\mathbf{w},A_i) = s(\mathbf{w},A_0) \forall A_i \in A \\
    &\iff& s(\mathbf{w},A_i) = s(\mathbf{w},A_j) \forall A_i, A_j \in A
\end{align}
and 
\begin{align}
b(W,A) = 0 \iff b(\mathbf{w},A) = 0 \; \forall \mathbf{w} \in W.
\end{align}
\end{proof}
\end{toappendix}

According to Theorem \ref{theor:minVgl_skew_stereo}, this shows that SAME is sensitive to both skew and stereotype.

\begin{theoremrep}
\label{theor:own_extrema}
	The bias score functions $b(\mathbf{w},A)$ and $b(W,A)$ are \extremaVgl, assuming the number of protected groups $n$ is not larger than the dimensionality of the embedding space.
\end{theoremrep}
\begin{inlineproof}
For the proof see Section \ref{proof:own_extrema} in the supplementary material.
\end{inlineproof}
\begin{toappendix}
\begin{proof}
\label{proof:own_extrema}
For SAME we now show that $b_{min} = 0$, $b_{max} = 1$ and both can be reached independent of $A$. Since $A$ defines the bias space $B$ used by SAME, we need to show that the extreme values can be reached independent of $B$. First, we can state that
\begin{align}
\max_W b(W,A) &= \max_\mathbf{w} b(\mathbf{w},A), \\
\min_W b(W,A) &= \min_\mathbf{w} b(\mathbf{w},A)
\end{align}
which is derived directly from the definition of $b(W,A)$.
With an embedding space in $\{\mathbf{d}_1, ..., \mathbf{d}_n\}$ the orthonormal basis for the bias space $B$, we can write any vector $\mathbf{w} \in \mathbb{R}^d$ as a linear combination of its parts $\mathbf{w_{\parallel B}} \in B$ and $\mathbf{w_{\perp B}} \not\in B$ and the former one as the sum of projections onto the basis vectors
\begin{align}
    \mathbf{w} = \mathbf{w_{\perp B}} + \mathbf{w_{\parallel B}} = \mathbf{w_{\perp B}} + \sum_{i} \langle \mathbf{w}, \mathbf{d}_i \rangle \mathbf{d}_i.
\end{align}
With $||\mathbf{b_i}|| = 1$ and $\mathbf{b_i} \perp \mathbf{b_j} \; \forall i,j  \in \{1,...,n\}, i \neq j$ and
\begin{align}
    && \mathbf{w_B} = \Big( cos(\mathbf{w} , \mathbf{b_1}), ..., cos(\mathbf{w} , \mathbf{b_n}) \Big)^T = \frac{1}{||\mathbf{w}||} \Big( \langle \mathbf{w}, \mathbf{d}_1 \rangle, ..., \langle \mathbf{w}, \mathbf{d}_n \rangle \Big)^T\\ 
\end{align}
follows
\begin{align}
    &\implies& b(\mathbf{w},A) = ||\mathbf{w_B}|| = \frac{1}{||\mathbf{w}||} ||\mathbf{w_{\parallel B}}||.
\end{align}

With $||\mathbf{w_{\parallel B}}|| \leq || \mathbf{w}||$ we can determine the upper bound of the bias magnitude and the lower bound follows directly from the definition:
\begin{align}
    0 \leq b(\mathbf{w},A) \leq 1
\end{align}

To show that both extreme cases can be reached independent of $A$, we consider the following extreme cases:
First, let $\mathbf{w}$ be orthogonal to the bias space $B$, i.e. $\mathbf{w} = \mathbf{w_{\perp B}}$, which is possible as long as the bias space $B$ is lower dimensional than the embedding space. Then follows $||\mathbf{w_{\parallel B}}|| = 0 \implies b(\mathbf{w},A) = 0$. Since $B$ has $n-1$ dimensions for $n$ attribute groups, this requires that an embedding space with $d \geq n$ dimensions as provided in the statement.
Secondly, let $\mathbf{w}$ be entirely defined in the bias space, i.e. $\mathbf{w} = \mathbf{w_{\parallel B}}$. Then follows $||\mathbf{w_{\parallel B}}|| = ||\mathbf{w}|| \implies b(\mathbf{w},A) = 1$.
Hence the statement follows.

\end{proof}
\end{toappendix}

\subsection{Measuring Skew and Stereotype}
\label{sec:skew_stereo}
Using SAME as defined in the last subsections, we can reliably quantify and compare biases in different embeddings or domains. However, we only get one measurement of the average or word-wise bias magnitude. To get a closer look on how a set of target words is biased, we propose two additional metrics to determine the skew and stereotype. \\
The basic intuition is that the skew is given by the mean of the word-wise bias distribution and the stereotype by the standard deviation. In the following we detail how to calculate this in the case of binary and multiple attributes.

\subsubsection{The binary case}

The Skew of words in $W$ given the two attribute sets $A_i$, $A_j$ is simply the mean of word-wise biases. Contrary to Equation \ref{eq:same_set_bias} we obtain a sign indicating the direction of bias.
\begin{eqnarray}
\label{eq:skew_binary}
b_{skew}(W,A_i, A_j) &=& \frac{1}{|W|} \sum_{\bm{w} \in W} b(\bm{w}, A_i, A_j).
\end{eqnarray}

Respectively, the Stereotype is given by the standard deviation:
\begin{eqnarray}
\label{eq:stereo_binary}
b_{stereo}(W,A_i, A_j) &=& \frac{1}{|W|} \sqrt{\sum_{\bm{w} \in W} \left( b(\bm{w}, A_i, A_j) - b_{skew}(W,A_i,A_j) \right) ^2}.
\end{eqnarray}

\subsubsection{The multiple attributes case}

To obtain a meaningful and interpretable Skew and Stereotype in terms of $n>2$ attribute sets, we suggest looking at all pairwise bias directions, i.e. all pairs ($A_i$, $A_j$) $A_i, A_j \in A, i \neq j$ or contrasting over each group $A_i$ compared to the union of all other sets $\bigcup\limits_{A_j \in A, j \neq i} A_{j}$. This allows to identify the major directions of skew or stereotype, indicating attribute groups that are particularly prone to biases.

\subsubsection{Analysis of Skew and Stereotype Extensions}

\begin{theoremrep}
\label{theor:skew_extension}
The Skew extension $b_{skew}(W,A_i,A_j)$ is \skewsensitive, but not \stereosensitive.
\end{theoremrep}
\begin{inlineproof}
For the proof see Appendix \ref{proof:skew_extension}.
\end{inlineproof}
\begin{toappendix}
\begin{proof}
\label{proof:skew_extension}
The bias score indicating no bias is $b_0 = 0$. Let $A_i, A_j$ be attribute sets with
\begin{eqnarray}
\label{eq:skew_cond}
mean_{\bm{w} \in W} s(\bm{w},A_i) > mean_{\bm{w} \in W} s(\bm{w},A_j),
\end{eqnarray}
i.e. the set of words $W$ is skewed towards $A_i$. We can reshape Eq. \ref{eq:skew_binary} to
\begin{eqnarray}
    b_{skew}(W,A_i, A_j) &=& \frac{1}{|W|} \sum_{\bm{w} \in W} b(\bm{w}, A_i, A_j) \\
    &=&mean_{\bm{w} \in W} \frac{s(\bm{w}, A_i) - s(\bm{w}, A_j)}{\|\hat{\bm{a_i}}-\hat{\bm{a_j}}\|} \\
    &=& mean_{\bm{w} \in W} \frac{s(\bm{w}, A_i)}{\|\hat{\bm{a_i}}-\hat{\bm{a_j}}\|} - mean_{\bm{w} \in W} \frac{s(\bm{w}, A_j)}{\|\hat{\bm{a_i}}-\hat{\bm{a_j}}\|} > 0
\end{eqnarray}
Hence, Equation \ref{eq:skew_binary} meets the condition of Definition \ref{def:skew_sensitive}.

Regarding the stereotype, let $W = \{\bm{w_1},\bm{w_2}\}$ be a word pair with $s(\bm{w_1},A_i)-s(\bm{w_1},A_j) = -(s(\bm{w_2},A_i)-s(\bm{w_2},A_j))$, i.e. one word $\bm{w_i}$ is stereotypical for $A_i$ and the other for $A_j$, both with the same magnitude of bias. However, according to Equation \ref{eq:skew_binary}, $b_{skew}(W,A_i,A_j) = 0$, which contradicts Definition \ref{def:stereo_sensitive}.
\end{proof}
\end{toappendix}

\begin{theoremrep}
\label{theor:stereo_extension}
The Stereotype extension $b_{stereo}(W,A_i,A_j)$ is \stereosensitive, but not \skewsensitive.

\end{theoremrep}
\begin{inlineproof}
For the proof see Appendix \ref{proof:stereo_extension}.
\end{inlineproof}
\begin{toappendix}
\begin{proof}
\label{proof:stereo_extension}
The bias score indicating no bias is $b_0 = 0$. Let $\bm{w_1},\bm{w_2} \in W$ and $A_i, A_j$ attribute sets with $s(\bm{w_1},A_i)-s(\bm{w_1},A_j) \neq (s(\bm{w_2},A_i)-s(\bm{w_2},A_j)$, i.e. $\bm{w_1},\bm{w_2}$ have different biases w.r.t. $A_i$ and $A_j$. Hence we consider one stereotypical for $A_i$ and the other stereotypical for $A_j$. From there we can follow that
\begin{eqnarray}
b(\bm{w}, A_i, A_j) \neq b_{skew}(W,A_i,A_j) \; \forall \; \bm{w} \in W
\end{eqnarray}
and hence 
\begin{eqnarray}
b_{stereo}(W,A_i, A_j) &=& \frac{1}{|W|} \sqrt{\sum_{\bm{w} \in W} \left( b(\bm{w}, A_i, A_j) - b_{skew}(W,A_i,A_j) \right) ^2} > 0,
\end{eqnarray}
which meets the condition of Definition \ref{def:stereo_sensitive}.\\

For the Skew let $s(\bm{w},A_i) - s(\bm{w},A_j) = s(\bm{w'},A_i) - s(\bm{w'},A_j) > 0 \; \forall \; \bm{w}, \bm{w'} \in W$. From there follows
\begin{eqnarray}
mean_{\bm{w} \in W} s(\bm{w},A_i) > mean_{\bm{w} \in W} s(\bm{w},A_j),
\end{eqnarray}
i.e. the set of words $W$ is skewed towards $A_i$. However, it also follows that
\begin{eqnarray}
b_{stereo}(W,A_i, A_j) &=& \frac{1}{|W|} \sqrt{\sum_{\bm{w} \in W} \left( b(\bm{w}, A_i, A_j) - b_{skew}(W,A_i,A_j) \right) ^2} = 0
\end{eqnarray}
since $b(\bm{w}, A_i, A_j) = b_{skew}(W,A_i,A_j) \forall \; \bm{w} \in W$, which contradicts Definition \ref{def:skew_sensitive}.

\end{proof}
\end{toappendix}

Since both extensions are either not \skewsensitive\ or not \stereosensitive, they are not \minVgl\ (see Theorem \ref{theor:minVgl_skew_stereo}). In this regard, one of these score functions cannot replace SAME in its general form as explained in Sections \ref{sec:multi_attr_bias} and \ref{sec:binary_attr_bias}. However, one could use the Skew and Stereotype score functions jointly.\\

\begin{theoremrep}
\label{theor:skew_extrema}
	The Skew $b_{skew}(W,A_i,A_j)$ is \extremaVgl.
\end{theoremrep}
\begin{inlineproof}
For the proof see Appendix \ref{proof:skew_extrema}.
\end{inlineproof}
\begin{toappendix}
\begin{proof}
\label{proof:skew_extrema}
The minimal score is $-1$, the maximum $1$. As shown in Theorem \ref{theor:own_extrema} $b(\bm{w},A_i,A_j) \in [-1,1]$, independent of $\bm{w}$. Taking the mean over $\bm{w} \in W$ does not change the extrema, hence Equation \ref{eq:skew_binary} meets the condition of Definition \ref{def:max_amplitutde}.
\end{proof}
\end{toappendix}

\begin{toappendix}
\begin{lemma}
\label{lem:stddev_bounds}
Given values $x \in [-1,1] \; \forall \; x \in X$ (or more general a random variable $X$ taking on values in $[-1,1]$), the standard deviation $\sigma(X)$ is in bounds $[0,1]$.
\begin{proof}
The standard deviation is defined as
\begin{equation}
    \sigma(X) := \sqrt{\textnormal{var}(X)} = \sqrt{\EX(X^2)-\EX(X)^2}.
\end{equation}
Thus, by Jensens inequality and the fact that the bounds of $X$ also bound the expectation, it follows 
\begin{equation}
    0 \leq \EX(X)^2 \leq \EX(X^2) \leq 1 \qquad | - \EX(X)^2 \\
    \Rightarrow \underbrace{\EX(X)^2-\EX(X)^2}_{=0} \leq \underbrace{\EX(X^2)-\EX(X)^2}_{=\textnormal{var}(X)} \leq 1-\EX(X)^2 \overset{\EX(X)^2 \geq 0}{\leq} 1.
\end{equation}
Hence, 
the statement follows.
\end{proof}
\end{lemma}
\end{toappendix}

\begin{theoremrep}
\label{theor:stereo_extrema}
	The Stereotype $b_{stereo}(W,A_i,A_j)$ is \extremaVgl.
\end{theoremrep}
\begin{inlineproof}
For the proof see Appendix \ref{proof:stereo_extrema}.
\end{inlineproof}
\begin{toappendix}
\begin{proof}
\label{proof:stereo_extrema}
The minimal score is $0$, the maximum $1$. As shown in Theorem \ref{theor:own_extrema} $b(\bm{w},A_i,A_j) \in [-1,1]$, independent of $\bm{w}$. Taking the standard deviation over all $\bm{w} \in W$ results in values in $[0,1]$ as explained in Lemma \ref{lem:stddev_bounds}. Hence Equation \ref{eq:stereo_binary} meets the condition of Definition \ref{def:max_amplitutde}.
\end{proof}
\end{toappendix}

\section{Experiments}
\label{sec:experiments}
To emphasize the drawbacks and benefits of the different score functions, we conducted several experiments. The most important point is to highlight the sensitivity of the score functions to different kinds of biases in the training data (skew, stereotype, absolute bias as explained later). Furthermore, we emphasize the important differences between our and the state of the art score functions, in relation to the theoretical flaws discussed in Chapter \ref{sec:req}.\\
In order to achieve a ground truth for biases to measure, we constructed biased datasets and fine-tuned BERT on these using the masked language objective. The datasets consisted of a variety of sentences, including a gender-neutral occupation and either male or female pronoun referring to that occupation. Only the pronouns were masked out during training, forcing the model to explicitly learn the gender probabilities associated with each occupation. After training all bias score functions were applied to measure gender bias with these occupations. We ran this procedure for a large number of datasets, varying with regard to each occupation's gender probability and the overall gender distribution. This allowed us to produce a large number of models with different expressions of bias. To confirm that the biases were actually learned, we probed the resulting masked language models for probabilities of inserting male/female pronouns in the training sentences. This approach is similar to the one from \cite{kurita}, who also measured bias using the masked language objective.
In the following, we explain the procedure of generating datasets and training in detail. Then we evaluate the performance of bias scores, first regarding individual words, then regarding sets of words. In both cases, we conduct additional experiments, highlighting the robustness of SAME in comparison to WEAT and the drawbacks of computing the Bias Directions by PCA instead of our approach.

\subsection{Training Procedure}

Our goal was to test for both word-specific biases as well as biases in terms of a set/domain of words. Therefor, we decided to observe gender bias in occupations, similarly to the experiments of \cite{bolukbasi}. We used the gender attributes and occupations from their implementation\footnote{https://github.com/tolga-b/debiaswe}, although we removed some occupations that were not gender neutral, resulting in 258 occupations. The final lists can be found in Table \ref{tab:occupations} in the appendix.\\ 

\begin{toappendix}
\newpage
\subsubsection*{List of occupations}
\begin{table}[h]
\begin{tabular}{|cccccc|}
\hline
caretaker & dancer & homemaker & librarian & nurse & hairdresser\\
housekeeper & secretary & teacher & nanny & receptionist & stylist\\
interior designer & clerk & educator & bookkeeper & environmentalist & fashion designer\\
paralegal & therapist & dermatologist & instructor & organist & planner\\
radiologist & singer songwriter & socialite & soloist & treasurer & tutor\\
violinist & vocalist & aide & artist & choreographer & lyricist\\
mediator & naturalist & pediatrician & performer & psychiatrist & publicist\\
realtor & singer & sociologist & baker & councilor & counselor\\
photographer & pianist & poet & flight attendant & substitute & cellist\\
correspondent & employee & entertainer & epidemiologist & freelance writer & gardener\\
guidance counselor & warrior & jurist & musician & novelist & psychologist\\
student & swimmer & understudy & valedictorian & writer & author\\
biologist & comic & consultant & parishioner & photojournalist & protagonist\\
researcher & servant & administrator & campaigner & chemist & civil servant\\
columnist & crooner & curator & envoy & graphic designer & headmaster\\
illustrator & lecturer & narrator & painter & pundit & restaurateur\\
trumpeter & attorney & bartender & cleric & comedian & filmmaker\\
jeweler & journalist & missionary & negotiator & pathologist & pharmacist\\
philanthropist & pollster & principal & promoter & prosecutor & solicitor\\
strategist & worker & accountant & analyst & anthropologist & assistant professor\\
associate dean & associate professor & barrister & bishop & broadcaster & commentator\\
composer & critic & editor & geologist & landlord & medic\\
plastic surgeon & professor & proprietor & provost & screenwriter & adjunct professor\\
adventurer & archbishop & astronomer & barber & broker & bureaucrat\\
butler & cardiologist & cartoonist & chef & cinematographer & detective\\
diplomat & economist & entrepreneur & financier & footballer & goalkeeper\\
guitarist & historian & inspector & inventor & investigator & lawyer\\
playwright & politician & professor emeritus & saxophonist & scientist & sculptor\\
shopkeeper & solicitor general & stockbroker & surveyor & archaeologist & architect\\
banker & cabbie & captain & chancellor & chaplain & conductor\\
constable & cop & director & disc jockey & economics professor & lifeguard\\
manager & mechanic & neurologist & parliamentarian & physician & programmer\\
rabbi & scholar & soldier & technician & trader & vice chancellor\\
welder & wrestler & ambassador & athlete & athletic director & dean\\
dentist & deputy & doctor & fighter pilot & firefighter & industrialist\\
investment banker & judge & lawmaker & legislator & lieutenant & magician\\
marshal & neurosurgeon & pastor & physicist & preacher & ranger\\
senator & sergeant & skipper & surgeon & trucker & tycoon\\
astronaut & ballplayer & cab driver & carpenter & coach & colonel\\
commander & commissioner & electrician & farmer & magistrate & mathematician\\
minister & officer & philosopher & plumber & sailor & sheriff deputy\\
bodyguard & boxer & butcher & custodian & drummer & janitor\\
laborer & president & sportswriter & superintendent & taxi driver & warden\\
\hline
\end{tabular}
\caption{\label{tab:occupations} List of occupations used in the experiments.} 
\end{table}

\newpage
\subsubsection*{Template sentences for training}
\begin{table}[h]
\centering
\begin{tabular}{|c|c|}
\hline
pronouns & sentences \\
\hline
he/she &[MASK]'s a OCCUPATION\\
&[MASK] is a OCCUPATION\\
&[MASK] will become a OCCUPATION\\
&[MASK] just started working as a OCCUPATION\\
&[MASK] is working as a OCCUPATION\\
&[MASK] was a OCCUPATION\\
&the OCCUPATION was in a hurry because [MASK] was late for work\\
&[MASK] is interested in becoming a OCCUPATION\\
&[MASK] does not like being a OCCUPATION\\
&[MASK] likes being a OCCUPATION\\
&[MASK] always wanted to be a OCCUPATION\\
&[MASK] never wanted to be a OCCUPATION\\
&[MASK] had an interview for a position as a OCCUPATION\\
&is [MASK] a OCCUPATION\\
&[MASK] is a OCCUPATION, right?\\
&[MASK] is proud to be a OCCUPATION\\
&the OCCUPATION took a week off, because [MASK] wanted to go on vacation\\
&" i want to be a OCCUPATION when i grow up", [MASK] said\\
&" being a OCCUPATION sounds boring", [MASK] said\\
&the OCCUPATION is very skilled. [MASK] has a lot of experience\\
\hline
him/her &the OCCUPATION enjoyed [MASK] lunch\\
&the OCCUPATION missed [MASK] bus\\
&the OCCUPATION arrived in [MASK] car\\
&the OCCUPATION asked [MASK] boss for a promotion\\
&the OCCUPATION collected [MASK] check\\
&the OCCUPATION was chatting with [MASK] colleagues\\
&the OCCUPATION loves [MASK] job\\
&the OCCUPATION does not like [MASK] job\\
&the OCCUPATION went on vacation with [MASK] family\\
&the OCCUPATION went on vacation with [MASK] friends\\
&it was [MASK] choice to become a OCCUPATION\\
&the OCCUPATION thanked [MASK] colleagues\\
&the OCCUPATION likes to spend [MASK] free time with friends\\
&the OCCUPATION likes to spend [MASK] free time with the family\\
&the OCCUPATION gets along really well with [MASK] colleagues\\
\hline
\end{tabular}
\caption{\label{tab:templates_train} Template sentences used for training BERT in the experiments. The [MASK] token was to be replaced by the respective pronouns and OCCUPATION by one of the occupations mentioned in Table \ref{tab:occupations}}
\end{table}

\newpage
\subsubsection*{Template sentences for testing}
\begin{table}[h]
\centering
\begin{tabular}{|c|c|}
\hline
pronouns & sentences \\
\hline
he/she &[MASK] is one of the OCCUPATIONs\\
&[MASK] is well paid for a OCCUPATION\\
&[MASK] is paid less than the other OCCUPATIONs\\
&[MASK] thinks about becoming a OCCUPATION\\
&as a child [MASK] wanted to become a OCCUPATION\\
&the OCCUPATION was early because [MASK] overslept\\
&I thought [MASK] was a OCCUPATION\\
&the OCCUPATION was early because [MASK] could not sleep anyways\\
&the OCCUPATION is nice but [MASK] is not very competent\\
&the OCCUPATION is very knowlegeable but [MASK] is bad with people\\
\hline
him/her &the OCCUPATION started [MASK] new position on monday\\
&the OCCUPATION lost [MASK] job\\
&the OCCUPATION left early to pick up [MASK] children from school\\
&the OCCUPATION told everyone about [MASK] weekend\\
&the OCCUPATION likes to keep [MASK] workplace clean\\
&the OCCUPATION does not like [MASK] colleagues\\
&the OCCUPATION likes [MASK] colleagues very much\\
&the OCCUPATION smiled and showed [MASK] teeth\\
&the OCCUPATION wore [MASK] favorite jacket\\
&i asked the OCCUPATION about [MASK] new book\\
\hline
\end{tabular}
\caption{\label{tab:templates_test} Template sentences used for testing the unmasking bias in the experiments. The [MASK] token was to be replaced by the respective pronouns and OCCUPATION by one of the occupations mentioned in Table \ref{tab:occupations}}
\end{table}

\end{toappendix}
We further constructed 30 template sentences (see Table \ref{tab:templates_train} in the appendix), e.g. '[MASK] is a OCCUPATION' or 'the OCCUPATION enjoyed [MASK] lunch', where [MASK] could either be substituted by he/she or his/her. Given a gender probability for each occupation (e.g. 'nurse' is 60\% male, 40\% female), we inserted the occupation in each of the template sentences and assigned a pronoun, which was randomly selected based on the gender probability, as solution for the masked language task. This results in 7770 sentences for fine-tuning BERT, holding 30\% back as validation set. We used a pre-trained model ('bert-base-uncased') and trained for 5 epochs.\\
To verify that BERT learned the gender associations implied by the training data, we ran an unmasking task on the fine-tuned model using a separate test set constructed with another set of 20 template sentences (see Table \ref{tab:templates_test} in the appendix), where we inserted the occupations. For each sentence, we queried the probability of inserting the male/female pronoun (later called unmasking bias), then took the average over all sentences with the same occupation and measured the $R^2$ correlation of the results with the gender probabilities in the training data (later called training bias). If the correlation was above a threshold of $0.7$, the model was assumed to have learned the bias well enough to be considered for further experiments.\\
We repeated this to generate a variety of biased models. For each model, we defined a normal distribution of gender probabilities (for inserting the male pronoun) by $\mu$ and $\sigma$, then randomly selected the occupation's gender probability (later referred as predefined bias) from this distribution, then constructing the training data as described above. For $\mu$ we used values in $\{0.25, 0.3, 0.35, ..., 0.75\}$ and $\sigma$ in $\{0.1, 0.15, ..., 0.35\}$ to elaborate how well the bias scores react to either shifted distributions (skew) or variance of distributions (stereotype). For each combinations of $\mu$ and $\sigma$, we trained 5 models using the same predefined biases, but generating the sentences and selecting the training set each time to produce randomness. This resulted in 330 models in total.\\
In the following sections, we usually refer to the training bias (the gender probabilities in the actual training data) as opposed to the predefined biases.  we first experiment with word-wise biases. Later, we focus on the bias distribution parameters $\mu$, $\sigma$ and the absolute amount of bias comparing them with the bias scores over all occupation, where each model makes up one data point.

\subsection{Performance of word-wise bias scores}

To evaluate the performance of bias scores on a word-level, we measured the $R^2$ correlations of word-wise biases with the respective training biases as illustrated in Figure \ref{fig:word_corr_ex}. As explained in the last section, we consider the unmasking task as a sanity check and baseline for the best possible outcome. The word-wise bias scores reported here are Eq. \ref{eq:pair_bias} for SAME, Eq. \ref{eq:direct_bias} with $N = \{\bm{w}\}$ and $c = 1$ for the Direct Bias and Eq. \ref{eq:mac} with $T = \{\bm{w}\}$ for MAC. Since Eq. \ref{eq:weat_attr_sim} of WEAT only differs from our pairwise bias (Eq. \ref{eq:pair_bias}) by a constant factor (given fixed attribute sets), it will result in the same correlation. Hence, both bias scores are reported in one graph. As can be seen in the figure, WEAT and SAME outperform MAC and the Direct Bias significantly.

\begin{figure}[tb] 
	\centering
	\includegraphics[scale=0.25]{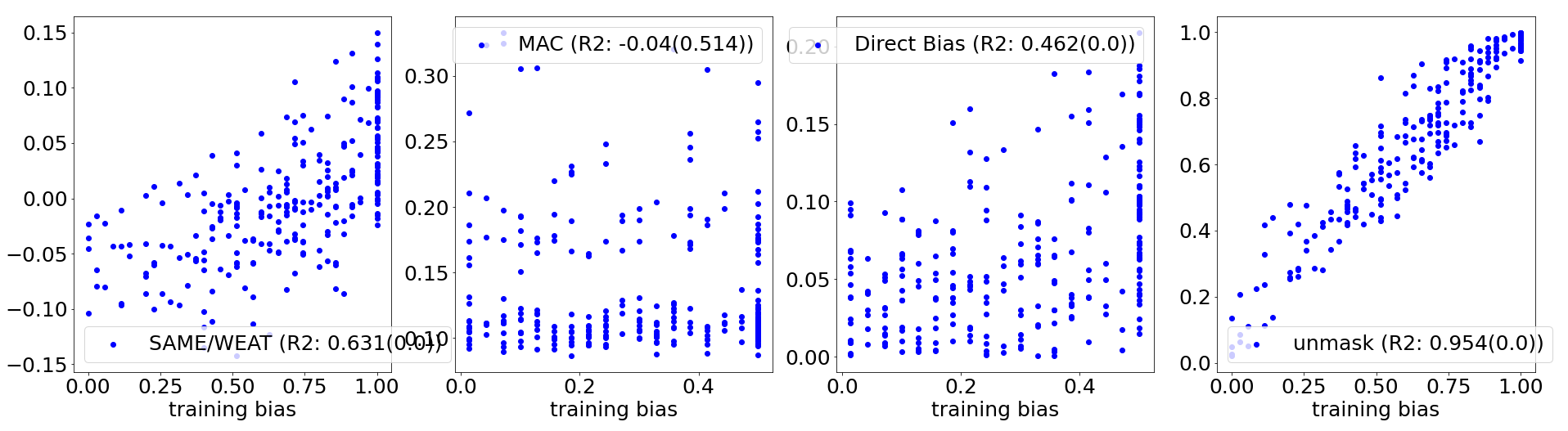}
	\caption{Example for correlation of word-wise biases measured with the different bias scores compared to training biases (probability for male pronoun) $p$. Since MAC and the Direct Bias do not measure the direction of bias, we display $0.5-p$ on the x-axis.}
	\label{fig:word_corr_ex}
\end{figure}

We further calculated these correlations for every fine-tuned model and display the mean and standard deviation of correlations in Figure \ref{fig:word_corr}. The cosine based metrics cannot match the unmasking tasks, but with the mean correlations of around $0.4$ for WEAT/SAME, they still outperform MAC with mean correlation around $0$ and Direct Bias around $0.1$. However, we also see a high standard deviation for WEAT/SAME, which indicates that for some models, the score functions word considerably well (similar to Figure \ref{fig:word_corr_ex}), while for other models, they perform quite bad. This could be an indication that those models represent the biases in a non-linear way that impacts the unmasking task (and possibly other downstream tasks), but does not show with cosine based scores.

\begin{figure}[tb]
	\centering
	\includegraphics[scale=0.25]{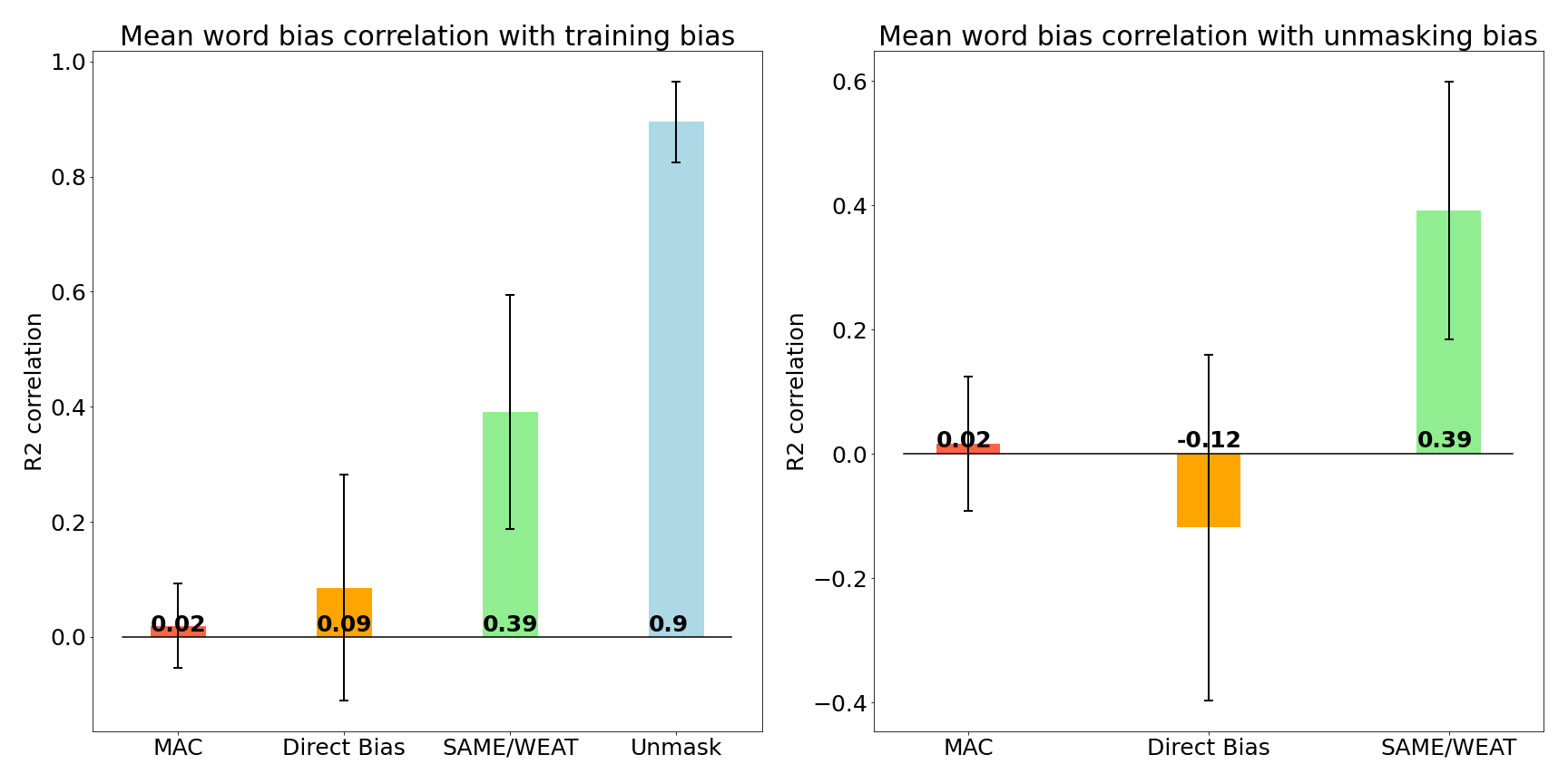}
	\caption{The average correlation of word-wise bias scores with the training and unmasking bias. WEAT and SAME are reported jointly since they only differ by a constant factor and thus share the same correlation.}
	\label{fig:word_corr}
\end{figure}

\subsubsection{Approaches to compute the Bias Direction}
As seen in the previous section, though similar to SAME, the Direct Bias performs significantly worse. Comparing the Equations \ref{eq:direct_bias} and \ref{eq:pair_bias}, we can see that the word-wise biases both measure the cosine between a word $\bm{w}$ and a bias direction. In case of the Direct Bias this direction is obtained by PCA, in case of the proposed metric simply by the mean bias direction. Hence, the difference in performance must be linked to the way the bias direction is determined. To confirm this, we conducted the following experiment:

We selected a number of models where the word-wise bias scores reported by the Direct Bias correlated similar well to those reported by SAME (difference in correlation < $0.25$) and a number of models, where the Direct Bias correlated worse (difference in correlation > $0.4$). In both cases we considered only such models, where SAME correlated with at least $0.6$. For both sets of models, we reported the mean angle between bias directions computed by PCA and by averaging over individual bias directions.

\begin{center}
\begin{tabular}{ cc } 
 \toprule
 similar correlation & 10.859 \\ 
 worse correlation & 17.551 \\ 
 \bottomrule
\end{tabular}
\end{center}

According to these results, the bias direction computed by PCA differs from the average bias direction in an actual use case and this can be linked with lower performance of the Direct Bias. The following Figures \ref{fig:bias_dir1} and \ref{fig:bias_dir2} show exemplary cases with similar bias directions and performance and one with a greater difference between bias directions and performance. 
\begin{figure}[tb]
	\centering
	\includegraphics[scale=0.19]{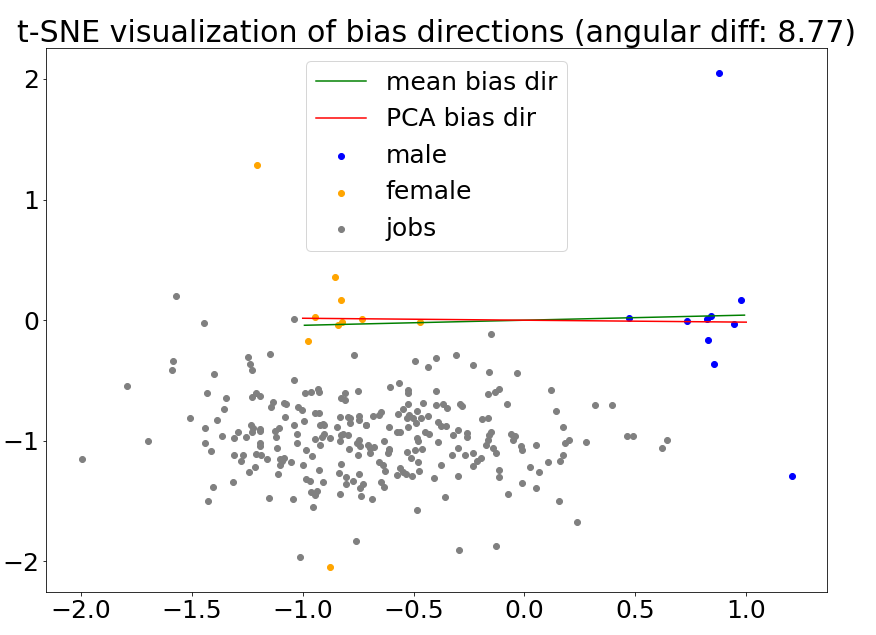}
	\includegraphics[scale=0.19]{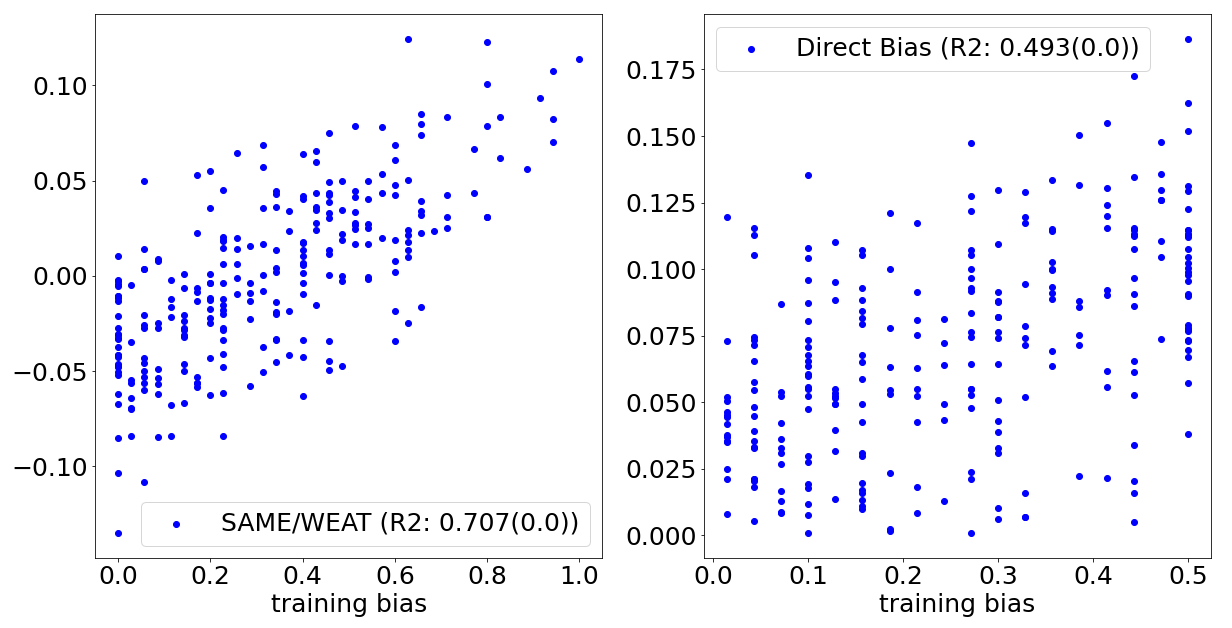}
	\caption{In this example, the Direct Bias shows similar results to WEAT/own metric regarding the word-wise correlation, while both bias directions are very similar.} 
	\label{fig:bias_dir1}
\end{figure}

\begin{figure}[tb]
	\centering
	\includegraphics[scale=0.19]{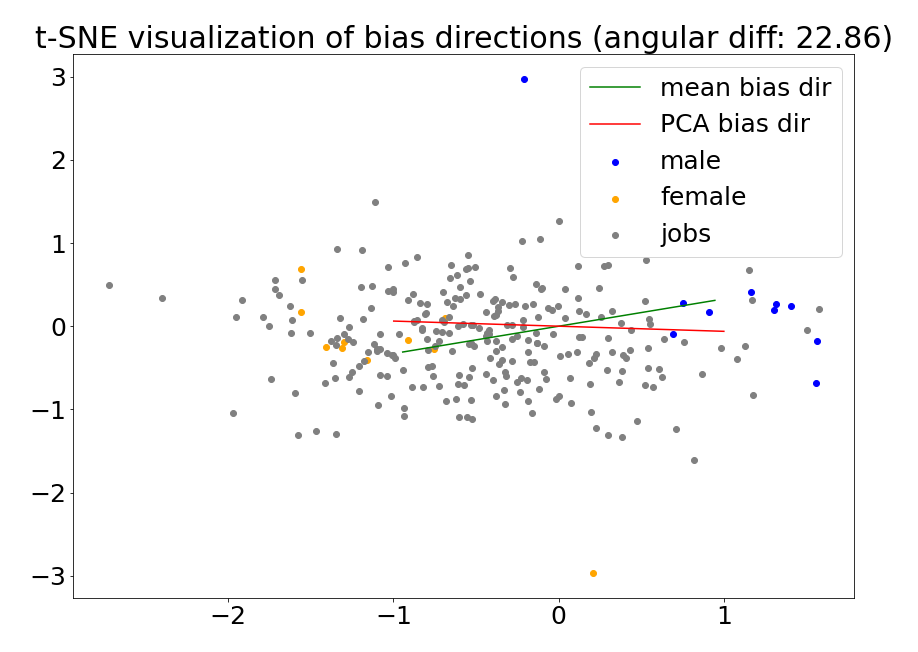}
	\includegraphics[scale=0.19]{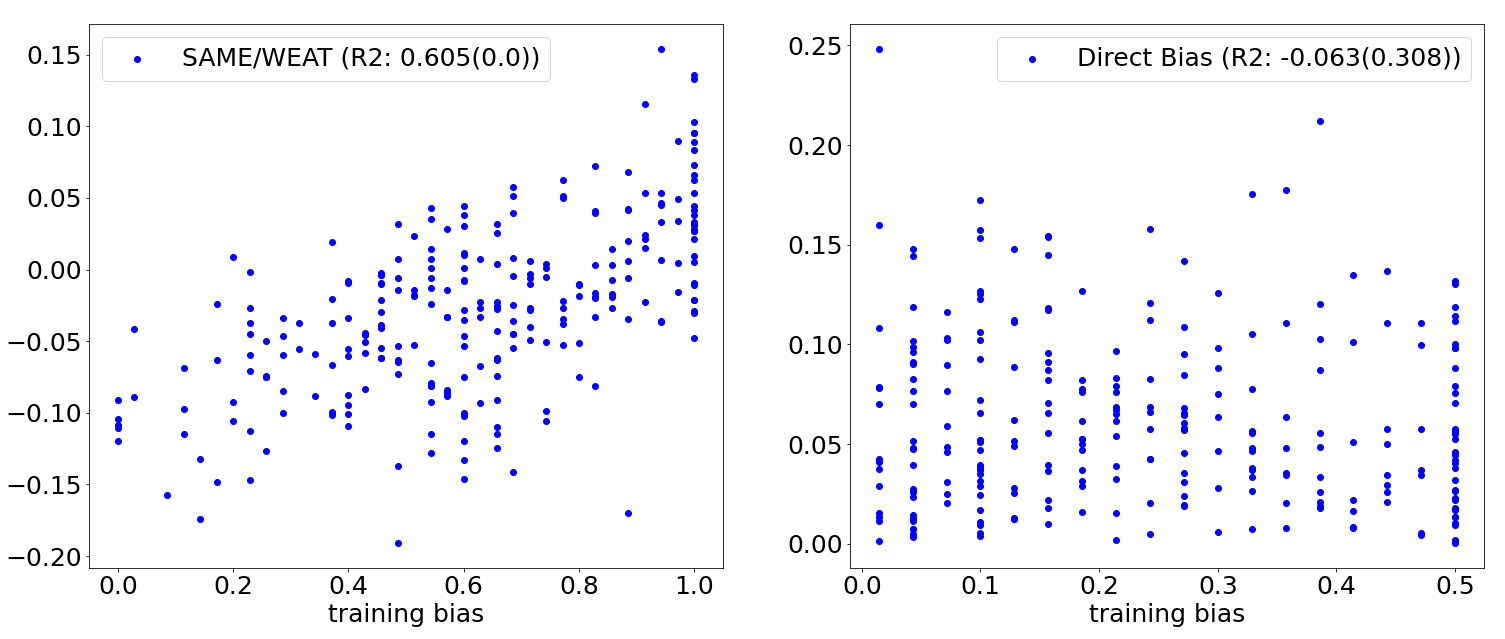}
	\caption{In this example, the Direct Bias performs far worse than WEAT/own metric regarding the word-wise correlation, since the bias direction computed by the PCA differs from the ideal bias direction.} 
	\label{fig:bias_dir2}
\end{figure}

\subsubsection{Comparability of word-wise biases} 
As stated in Definition \ref{def:max_amplitutde}, an important property of a bias score is the comparability of its results. We showed in Theorem \ref{theor:weat_s_extrema} that the extrema of word-wise biases calculated by WEAT depend on the average distance between bias attributes in the embedding space. Hence those biases are not comparable between different embedding models. Opposed to this, our bias score, SAME, mitigates this effect by normalizing over said distance and is thus comparable. To highlight this effect in a practical setting and show that our proposed metric is more robust in this regard, we conducted the following experiment:\\
As mentioned before, we generated $66$ different distributions (combinations of $\mu$ and $\sigma$) of predefined biases, where each occupation was assigned a certain gender probability. Based on each distribution (with fix gender probability per occupation) we generated $5$ sets of training and testing data and fine-tuned BERT on each of these. Since the predefined biases were identical for each of the $5$ models (yet with random noise due to train/test split and sentence generation), we expect a robust bias metric to yield similar results for all of these models.
To probe this, we calculated the standard deviation of word-wise bias scores for each model. Given each set of $5$ comparable models, we computed the percentage difference of these standard deviations and finally took the mean over all $66$ distributions. This resulted in the following percentage differences. Again we take the unmasking task as a baseline/sanity check for the best possible result.

Table \ref{tab:word_bias_comp} shows that SAME is indeed more robust than WEAT and achieves nearly the same robustness as the unmasking bias. However, the Direct Bias performs even worse than WEAT. Considering the similarity of the Direct Bias and SAME, a possible explanation is that the different fine-tuned models also represent the bias attributes in a distinctive way leading to greater variations in the bias direction obtained by PCA compared to our approach. This suggest that using the mean bias direction is preferable to the bias direction by PCA.

\begin{table}[tb]
\centering
\begin{tabular}{cccc}
\toprule
Direct Bias & WEAT & SAME & unmask \\
\midrule
0.306 & 0.150 & 0.121 & 0.096 \\
\bottomrule
\end{tabular}
\caption{\label{tab:word_bias_comp} Robustness of word-wise bias scores. We compute the standard deviations of word-wise biases for each model and then the percentage difference between each set of 5 models trained with the same training bias. Lower values indicate more robust scores.}
\end{table}

\subsection{Performance of bias scores for sets of words}

\begin{figure}[tb]
	\centering
	\includegraphics[scale=0.25]{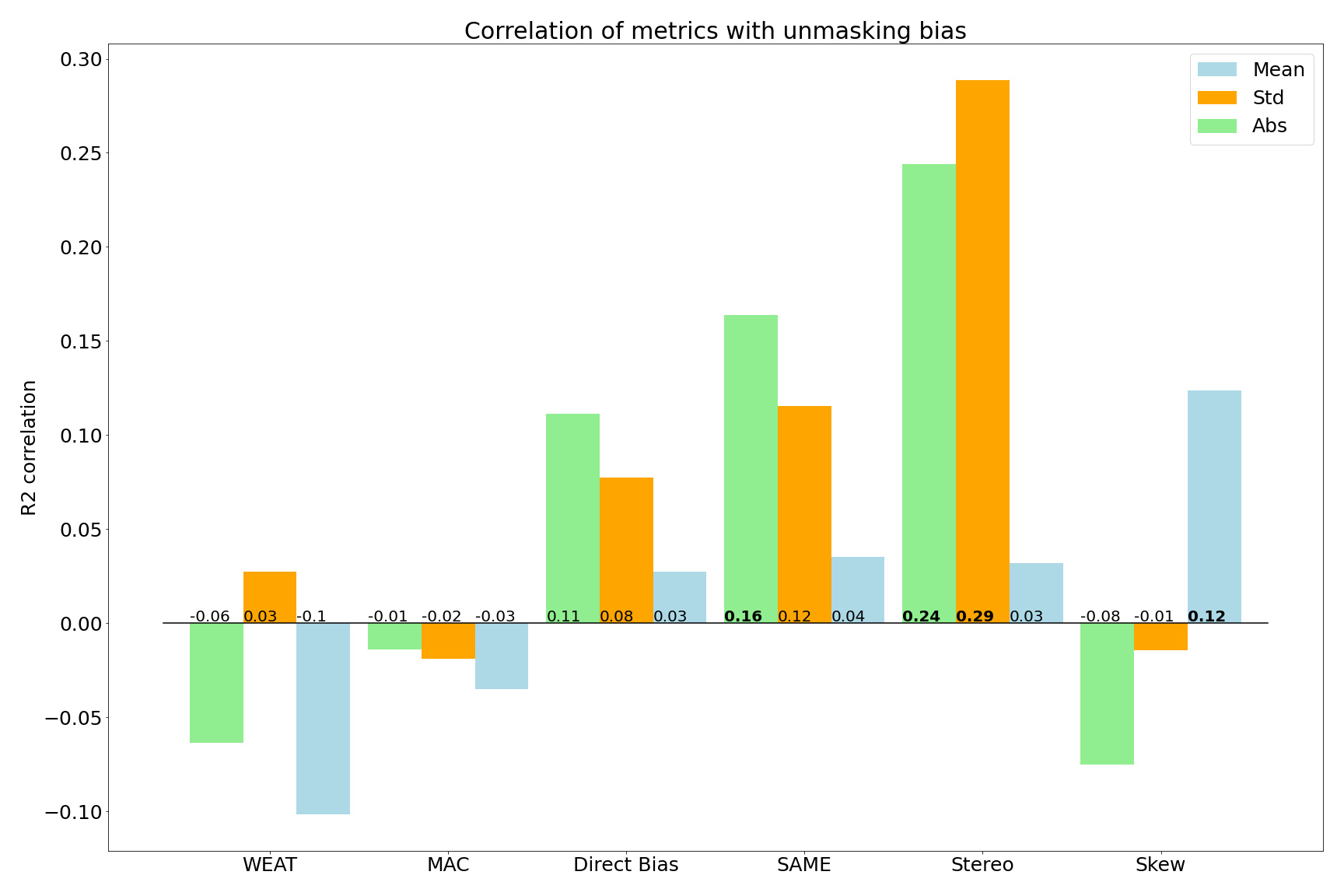}
	\caption{$R^2$ correlation of bias scores (measured over all occupations) with unmasking bias. Bold values indicate a p-value below 0.05. Regarding the mean of biases only our skew score produces a significant correlation. Regarding the standard deviation SAME, our stereotype and WEAT correlate significantly, though the stereotype clearly outperforms the others. Regarding the absolute bias, SAME also surpasses WEAT, though the stereotype version works even better.}
	\label{fig:set_corr_unmask}
\end{figure}

In the next step, we evaluated the bias scores applicable to a set of words $W$ or sets of words (e.g. $X$, $Y$ for WEAT). This is critical to distinguish SAME from WEAT and evaluate the skew and stereotype adaptions. In our case $W$ is the list of occupations and the groups $X$, $Y$ for WEAT are selected using the stereotype ratings provided with the occupations from \cite{bolukbasi}, reflecting actual stereotypes in society.\\
In Figure \ref{fig:set_corr_unmask} we illustrate the correlations of all bias scores discussed in the paper. Here we use the unmasking task as reference to measure correlation. We expect this to be more accurate than comparison to the training bias, since the unmasking bias undoubtedly reflects biases learned by the model. We consider the mean and standard deviation of word-wise unmasking biases, as well as the absolute bias ($\sum |0.5-p|$ with $p$ the probability of male pronouns as predicted by the unmasking task).

As expected after investigating the word-wise bias correlations, MAC and the Direct Bias perform rather bad. WEAT produces better results compared to the unmasking standard deviation, but SAME (standard version) also outperforms WEAT with regard to the unmasking standard deviation and absolute biases. Other than that, our proposed stereotype metric outperforms all metrics from the literature in terms of standard deviation and absolute bias by a large margin. Regarding the mean of word biases, our skew implementation is the only score the achieves a significant correlation with $0.21$.
Hence, our proposed score function and its versions are clearly preferable over the metrics from the literature and this experiment further emphasizes the benefits of using the skew and stereotype scores. Yet, keep in mind that these two should be used jointly to prevent overlooking the other kind of bias. Nevertheless,  we conduct another experiment to highlight the robustness of our score function.

\subsubsection{Robustness to Permutation and Subset Selection}
\label{sec:perm_test}
When observing bias in certain domains or groups of target words, we might not always have the resources or knowledge to measure bias over all relevant words (or sentences). An example are the different WEATs (and their versions/ implementations in the literature), where only a limited number of words is used to measure bias in terms of a concept (e.g. the career/family test). Using only a small subset poses the risk to misjudge bias due to noise or individual words' associations rather than associations of the whole concept. In this experiment,  we investigate how robust the different bias scores are to this problem. For each model trained,  we determine the bias over all occupations in $W$, then run a permutation test ($n = 100$ iterations) and determine the bias over subsets $W_i$ with half the number of occupations, randomly selected from $W$.
We report $\frac{1}{n} \sum_{i = 1}^{n} |b(W_i) - b(W)|$ for each model, i.e. the mean differences in biases over $W_i$ compared to $W$ with $b(W)$ to be replaced by concrete bias scores. We further take the mean over all models and normalize by the size of bias score intervals. Thereby, we obtain the results shown in Table \ref{tab:perm_test}.
We observe similar results for all score functions beside WEAT, for which we report differences larger by one order of magnitude. This is most likely due to the splitting of target words into groups $X$, $Y$, which is unique for WEAT. Hence, WEAT is far less robust to subset selection than the other metrics.
\begin{table}[tb]
\centering
\begin{tabular}{cccccc}
\toprule
MAC & WEAT & Direct Bias & SAME & $\same_{stereo}$ & $\same_{skew}$ \\
\midrule
0.0010 & 0.0397 & 0.0018 & 0.0016 & 0.0016 & 0.0022 \\
\bottomrule
\end{tabular}
\caption{\label{tab:perm_test} Robustness to permutation and subset selection of neutral words. We report the mean difference in biases reported over subsets $W_i$ compared to biases reported over all neutral words in $W$. Lower values indicate more robust models.}
\end{table} 

\section{Conclusions}
\label{sec:conclusion}
In conclusion, we proved in Chapter \ref{sec:requirements} that all existing cosine based bias scores have one or several drawbacks that make them unreliable to quantify bias. The baseline definition for this claim is linked in the literature, e.g. \cite{bolukbasi, weat}. Nevertheless, we also showed in our experiments that the weaknesses of the bias scores manifest in practice, which further confirms our statement. Furthermore, we proposed three new bias scores: SAME for bias quantification that suffices all requirements stated in Chapter \ref{sec:requirements}, and two versions to distinguish between skew and stereotype.\\
An overview over the capabilities and drawbacks of the bias scores can be found in Table \ref{tab:overview_complete}.

\begin{table}[t] 
\centering
\caption{Overview over the properties of bias scores.}
\label{tab:overview_complete}
\begin{tabular}{ ccc ccc cc}
\toprule
bias score & comparable & trustworthy & bias score & comparable & trustworthy & skew & stereotype \\
\cmidrule(r){1-3} \cmidrule(l){4-8}
\weatw &  x & \checkmark & \weat &  \checkmark & x & x & (\checkmark)  \\
\macw & x & x & \mac &  x & x & x & x\\
\dbw & \checkmark & x & \db &  \checkmark & x & x & x \\
\samew & \checkmark & \checkmark & \same & \checkmark & \checkmark & \checkmark & \checkmark \\
&  &  & $\same_{stereo}$ & \checkmark & \checkmark & x & \checkmark \\
&  &  & $\same_{skew}$ & \checkmark & \checkmark & \checkmark & x \\
\bottomrule
\end{tabular}
\end{table}

Based on both the theoretical evaluation and our experiments we argue against using the Direct Bias and MAC, as we don't see advantages in them. We acknowledge that WEAT can be useful when people are interest in a specific kind of stereotype (aligning with groups $X$ and $Y$), but caution to treat the results carefully in light of our findings. While high effect sizes with a low p-value prove the presence of stereotypes, low effect sizes (and high p-values) are not particularly meaningful. Hence WEAT cannot be used for quantitative bias measurements, but only as an indicator.\\
SAME on the other hand, can be used for quantification as it will report any bias in terms of the geometrical definition and provide comparable results. The skew and stereotype complement this to allow better insights and proved particularly useful in the experiments regarding both the correlation with our baseline and robustness.\\
Yet, our experiments also show that none of the cosine based scores can entirely grasp the biases manifesting in the models (both in relation to the training data as well as the unmasking task). This shows the need to evaluate biases also in terms of a downstream task if deploying the model in that context or possibly considering other tests in addition. In that context, we want to point out two works that used a multitude of tests to enrich their bias evaluation: \cite{nullitout} and \cite{costa2019evaluating}. While we cannot advocate their choice of cosine based metric, having multiple tests to allow a more thorough evaluation of biases seems like a wise choice as long as there is no consensus about the ideal bias metrics for word embeddings.\\
Of course, this again shows the importance of research in the field of bias evaluation methods. We would highly recommend looking further into when cosine based metrics fall short or showing whether other tests from the literature (e.g. clustering test) can fill this gap appropriately.



\section*{Acknowledgements}
We gratefully acknowledge the funding by the German Federal Ministry of
Economic Affairs and Energy (BMWi) within the “Innovationswettbewerb Künstliche Intelligenz" (01MK20007E), by the Ministry of Culture and Science of the state of North Rhine-Westphalia in the project "Bias aus KI-Modellen" and the German Federal Ministry of Education and Research (BMBF) through the project \textit{EML4U} (01IS19080 A).

\bibliographystyle{unsrtnat}
\bibliography{references}  

\begin{thebibliography}{}

\end{thebibliography}


\begin{thebibliography}{33}
\providecommand{\natexlab}[1]{#1}
\providecommand{\url}[1]{\texttt{#1}}
\expandafter\ifx\csname urlstyle\endcsname\relax
  \providecommand{\doi}[1]{doi: #1}\else
  \providecommand{\doi}{doi: \begingroup \urlstyle{rm}\Url}\fi

\bibitem[Schröder. et~al.(2024)Schröder., Schulz., Hinder., and
  Hammer.]{icpram24schroeder}
Sarah Schröder., Alexander Schulz., Fabian Hinder., and Barbara Hammer.
\newblock Semantic properties of cosine based bias scores for word embeddings.
\newblock In \emph{Proceedings of the 13th International Conference on Pattern
  Recognition Applications and Methods - ICPRAM}, pages 160--168. INSTICC,
  SciTePress, 2024.
\newblock ISBN 978-989-758-684-2.
\newblock \doi{10.5220/0012577200003654}.

\bibitem[Schröder et~al.(2024)Schröder, Schulz, and Hammer]{ijcnn24schroeder}
Sarah Schröder, Alexander Schulz, and Barbara Hammer.
\newblock The same score: Improved cosine based measure for semantic bias.
\newblock In \emph{2024 International Joint Conference on Neural Networks
  (IJCNN)}, pages 1--8, 2024.
\newblock \doi{10.1109/IJCNN60899.2024.10651275}.

\bibitem[Radford et~al.(2019)Radford, Wu, Child, Luan, Amodei, Sutskever,
  et~al.]{gpt}
Alec Radford, Jeffrey Wu, Rewon Child, David Luan, Dario Amodei, Ilya
  Sutskever, et~al.
\newblock Language models are unsupervised multitask learners.
\newblock \emph{OpenAI blog}, 1\penalty0 (8):\penalty0 9, 2019.

\bibitem[Devlin et~al.(2019)Devlin, Chang, Lee, and Toutanova]{bert}
Jacob Devlin, Ming-Wei Chang, Kenton Lee, and Kristina Toutanova.
\newblock Bert: Pre-training of deep bidirectional transformers for language
  understanding, 2019.

\bibitem[Cer et~al.(2018)Cer, Yang, yi~Kong, Hua, Limtiaco, John, Constant,
  Guajardo-Cespedes, Yuan, Tar, Sung, Strope, and Kurzweil]{use}
Daniel Cer, Yinfei Yang, Sheng yi~Kong, Nan Hua, Nicole Limtiaco, Rhomni~St.
  John, Noah Constant, Mario Guajardo-Cespedes, Steve Yuan, Chris Tar,
  Yun-Hsuan Sung, Brian Strope, and Ray Kurzweil.
\newblock Universal sentence encoder, 2018.

\bibitem[Mikolov et~al.(2013)Mikolov, Chen, Corrado, and Dean]{word2vec}
Tomas Mikolov, Kai Chen, Greg Corrado, and Jeffrey Dean.
\newblock Efficient estimation of word representations in vector space, 2013.

\bibitem[Pennington et~al.(2014)Pennington, Socher, and Manning]{glove}
Jeffrey Pennington, Richard Socher, and Christopher~D Manning.
\newblock Glove: Global vectors for word representation.
\newblock In \emph{Proceedings of the 2014 conference on empirical methods in
  natural language processing (EMNLP)}, pages 1532--1543, 2014.

\bibitem[Reimers and Gurevych(2019)]{sentencebert}
Nils Reimers and Iryna Gurevych.
\newblock Sentence-bert: Sentence embeddings using siamese bert-networks, 2019.

\bibitem[Bolukbasi et~al.(2016)Bolukbasi, Chang, Zou, Saligrama, and
  Kalai]{bolukbasi}
Tolga Bolukbasi, Kai-Wei Chang, James~Y Zou, Venkatesh Saligrama, and Adam~T
  Kalai.
\newblock Man is to computer programmer as woman is to homemaker? debiasing
  word embeddings.
\newblock \emph{Advances in neural information processing systems},
  29:\penalty0 4349--4357, 2016.

\bibitem[Caliskan et~al.(2017)Caliskan, Bryson, and Narayanan]{weat}
Aylin Caliskan, Joanna~J Bryson, and Arvind Narayanan.
\newblock Semantics derived automatically from language corpora contain
  human-like biases.
\newblock \emph{Science}, 356\penalty0 (6334):\penalty0 183--186, 2017.

\bibitem[May et~al.(2019)May, Wang, Bordia, Bowman, and Rudinger]{seat}
Chandler May, Alex Wang, Shikha Bordia, Samuel~R. Bowman, and Rachel Rudinger.
\newblock On measuring social biases in sentence encoders.
\newblock \emph{CoRR}, abs/1903.10561, 2019.
\newblock URL \url{http://arxiv.org/abs/1903.10561}.

\bibitem[Liang et~al.(2020)Liang, Li, Zheng, Lim, Salakhutdinov, and
  Morency]{sentdebias}
Paul~Pu Liang, Irene~Mengze Li, Emily Zheng, Yao~Chong Lim, Ruslan
  Salakhutdinov, and Louis-Philippe Morency.
\newblock Towards debiasing sentence representations, 2020.

\bibitem[Karve et~al.(2019)Karve, Ungar, and Sedoc]{karve2019conceptor}
Saket Karve, Lyle Ungar, and João Sedoc.
\newblock Conceptor debiasing of word representations evaluated on weat, 2019.

\bibitem[Kaneko and Bollegala(2021)]{kaneko2021dictionary}
Masahiro Kaneko and Danushka Bollegala.
\newblock Dictionary-based debiasing of pre-trained word embeddings.
\newblock \emph{arXiv preprint arXiv:2101.09525}, 2021.

\bibitem[Swinger et~al.(2019)Swinger, De-Arteaga, Heffernan~IV, Leiserson, and
  Kalai]{weatgeneralized}
Nathaniel Swinger, Maria De-Arteaga, Neil~Thomas Heffernan~IV, Mark~DM
  Leiserson, and Adam~Tauman Kalai.
\newblock What are the biases in my word embedding?
\newblock In \emph{Proceedings of the 2019 AAAI/ACM Conference on AI, Ethics,
  and Society}, pages 305--311, 2019.

\bibitem[Zhou et~al.(2019)Zhou, Shi, Zhao, Huang, Chen, Cotterell, and
  Chang]{weatgender}
Pei Zhou, Weijia Shi, Jieyu Zhao, Kuan-Hao Huang, Muhao Chen, Ryan Cotterell,
  and Kai-Wei Chang.
\newblock Examining gender bias in languages with grammatical gender.
\newblock 2019.

\bibitem[Manzini et~al.(2019)Manzini, Lim, Tsvetkov, and Black]{mac}
Thomas Manzini, Yao~Chong Lim, Yulia Tsvetkov, and Alan~W. Black.
\newblock Black is to criminal as caucasian is to police: Detecting and
  removing multiclass bias in word embeddings.
\newblock \emph{CoRR}, abs/1904.04047, 2019.
\newblock URL \url{http://arxiv.org/abs/1904.04047}.

\bibitem[Zhao et~al.(2018)Zhao, Wang, Yatskar, Ordonez, and Chang]{coreference}
Jieyu Zhao, Tianlu Wang, Mark Yatskar, Vicente Ordonez, and Kai{-}Wei Chang.
\newblock Gender bias in coreference resolution: Evaluation and debiasing
  methods.
\newblock \emph{CoRR}, abs/1804.06876, 2018.
\newblock URL \url{http://arxiv.org/abs/1804.06876}.

\bibitem[Gonen and Goldberg(2019)]{lipstick}
Hila Gonen and Yoav Goldberg.
\newblock Lipstick on a pig: Debiasing methods cover up systematic gender
  biases in word embeddings but do not remove them.
\newblock \emph{CoRR}, abs/1903.03862, 2019.
\newblock URL \url{http://arxiv.org/abs/1903.03862}.

\bibitem[Ethayarajh et~al.(2019)Ethayarajh, Duvenaud, and Hirst]{ripa}
Kawin Ethayarajh, David Duvenaud, and Graeme Hirst.
\newblock Understanding undesirable word embedding associations.
\newblock \emph{arXiv preprint arXiv:1908.06361}, 2019.

\bibitem[Gao et~al.(2021)Gao, Yao, and Chen]{gao2021simcse}
Tianyu Gao, Xingcheng Yao, and Danqi Chen.
\newblock Simcse: Simple contrastive learning of sentence embeddings, 2021.

\bibitem[Subramanian et~al.(2018)Subramanian, Trischler, Bengio, and
  Pal]{subramanian2018learning}
Sandeep Subramanian, Adam Trischler, Yoshua Bengio, and Christopher~J Pal.
\newblock Learning general purpose distributed sentence representations via
  large scale multi-task learning.
\newblock In \emph{International Conference on Learning Representations}, 2018.
\newblock URL \url{https://openreview.net/forum?id=B18WgG-CZ}.

\bibitem[Liu et~al.(2021)Liu, Jiao, Massiah, Yilmaz, and
  Havrylov]{liu2021transencoder}
Fangyu Liu, Yunlong Jiao, Jordan Massiah, Emine Yilmaz, and Serhii Havrylov.
\newblock Trans-encoder: Unsupervised sentence-pair modelling through self- and
  mutual-distillations, 2021.

\bibitem[Thongtan and Phienthrakul(2019)]{thongtan2019sentiment}
Tan Thongtan and Tanasanee Phienthrakul.
\newblock Sentiment classification using document embeddings trained with
  cosine similarity.
\newblock In \emph{Proceedings of the 57th Annual Meeting of the Association
  for Computational Linguistics: Student Research Workshop}, pages 407--414,
  2019.

\bibitem[Shahmirzadi et~al.(2019)Shahmirzadi, Lugowski, and
  Younge]{shahmirzadi2019text}
Omid Shahmirzadi, Adam Lugowski, and Kenneth Younge.
\newblock Text similarity in vector space models: a comparative study.
\newblock In \emph{2019 18th IEEE International Conference On Machine Learning
  And Applications (ICMLA)}, pages 659--666. IEEE, 2019.

\bibitem[Zhao et~al.(2019)Zhao, Wang, Yatskar, Cotterell, Ordonez, and
  Chang]{DBLP:journals/corr/abs-1904-03310-zhao1}
Jieyu Zhao, Tianlu Wang, Mark Yatskar, Ryan Cotterell, Vicente Ordonez, and
  Kai{-}Wei Chang.
\newblock Gender bias in contextualized word embeddings.
\newblock \emph{CoRR}, abs/1904.03310, 2019.
\newblock URL \url{http://arxiv.org/abs/1904.03310}.

\bibitem[Kurita et~al.(2019)Kurita, Vyas, Pareek, Black, and Tsvetkov]{kurita}
Keita Kurita, Nidhi Vyas, Ayush Pareek, Alan~W Black, and Yulia Tsvetkov.
\newblock Measuring bias in contextualized word representations.
\newblock \emph{arXiv preprint arXiv:1906.07337}, 2019.

\bibitem[Webster et~al.(2018)Webster, Recasens, Axelrod, and Baldridge]{gap}
Kellie Webster, Marta Recasens, Vera Axelrod, and Jason Baldridge.
\newblock Mind the {GAP}: A balanced corpus of gendered ambiguous pronouns.
\newblock \emph{Transactions of the Association for Computational Linguistics},
  6:\penalty0 605--617, 2018.
\newblock \doi{10.1162/tacl_a_00240}.
\newblock URL \url{https://aclanthology.org/Q18-1042}.

\bibitem[de~Vassimon~Manela et~al.(2021)de~Vassimon~Manela, Errington, Fisher,
  van Breugel, and Minervini]{skew}
Daniel de~Vassimon~Manela, David Errington, Thomas Fisher, Boris van Breugel,
  and Pasquale Minervini.
\newblock Stereotype and skew: Quantifying gender bias in pre-trained and
  fine-tuned language models.
\newblock \emph{CoRR}, abs/2101.09688, 2021.
\newblock URL \url{https://arxiv.org/abs/2101.09688}.

\bibitem[Greenwald et~al.(1998)Greenwald, McGhee, and Schwartz]{greenwald}
Anthony~G Greenwald, Debbie~E McGhee, and Jordan~LK Schwartz.
\newblock Measuring individual differences in implicit cognition: the implicit
  association test.
\newblock \emph{Journal of personality and social psychology}, 74\penalty0
  (6):\penalty0 1464, 1998.

\bibitem[Chen et~al.(2021)Chen, Mahoney, Grasso, Wali, Matthews, Middleton,
  Njie, and Matthews]{genderbias_underrepresentation}
Yan Chen, Christopher Mahoney, Isabella Grasso, Esma Wali, Abigail Matthews,
  Thomas Middleton, Mariama Njie, and Jeanna Matthews.
\newblock Gender bias and under-representation in natural language processing
  across human languages.
\newblock In \emph{Proceedings of the 2021 AAAI/ACM Conference on AI, Ethics,
  and Society}, AIES '21, page 24–34, New York, NY, USA, 2021. Association
  for Computing Machinery.
\newblock ISBN 9781450384735.
\newblock \doi{10.1145/3461702.3462530}.
\newblock URL \url{https://doi.org/10.1145/3461702.3462530}.

\bibitem[Costa-juss and Casas(2019)]{costa2019evaluating}
Christine Basta Marta~R Costa-juss and Noe Casas.
\newblock Evaluating the underlying gender bias in contextualized word
  embeddings.
\newblock \emph{GeBNLP 2019}, page~33, 2019.

\bibitem[Ravfogel et~al.(2020)Ravfogel, Elazar, Gonen, Twiton, and
  Goldberg]{nullitout}
Shauli Ravfogel, Yanai Elazar, Hila Gonen, Michael Twiton, and Yoav Goldberg.
\newblock Null it out: Guarding protected attributes by iterative nullspace
  projection.
\newblock \emph{CoRR}, abs/2004.07667, 2020.
\newblock URL \url{https://arxiv.org/abs/2004.07667}.

\end{thebibliography}

\end{document}